\documentclass[draft]{tlp}
\newcommand\bcmdtab{\noindent\bgroup\tabcolsep=0pt%
  \begin{tabular}{@{}p{10pc}@{}p{20pc}@{}}}
\newcommand\ecmdtab{\end{tabular}\egroup}

  \title[Theory and Practice of Logic Programming]
        {Rewriting recursive aggregates in answer set programming: back to monotonicity}

  \author[M. Alviano and W. Faber and M. Gebser]
         {MARIO ALVIANO \\
         University of Calabria, Italy % \\
         \and
         WOLFGANG FABER \\
         University of Huddersfield, UK % \\
         \and
         MARTIN GEBSER \\
         Aalto University, HIIT, Finland}

\jdate{March 2003}
\pubyear{2003}
\pagerange{\pageref{firstpage}--\pageref{lastpage}}
\doi{XXX}

\usepackage{amsmath}
\usepackage{amssymb} 
\usepackage{comment}
\usepackage{xspace}
\usepackage{mathtools}
\usepackage{url}
\usepackage{thmtools}
\usepackage{thm-restate}
\usepackage{helvet,times,courier}
\usepackage{enumitem}
\usepackage{tikz}\usetikzlibrary{arrows,fit,shapes.symbols}
\usepackage{multirow}

\newtheorem{proposition}{Proposition}

\newtheorem{examplee}{Example}
\newenvironment{example}{\begin{examplee}}{\hfill $\blacksquare$\end{examplee}}

\newcommand{\TODO}[1]{\textbf{[TODO:} #1\textbf{]}}
\newcommand{\naf}{\ensuremath{\raise.17ex\hbox{\ensuremath{\scriptstyle\mathtt{\sim}}}}\xspace} \DeclarePairedDelimiter\norm{\lVert}{\rVert}

\newif\ifInput
\Inputtrue % \Inputfalse

\def\V{\ensuremath{\mathcal{V}}\xspace}
\def\SUM{\ensuremath{\textsc{sum}}\xspace}
\def\COUNT{\ensuremath{\textsc{count}}\xspace}
\def\AVG{\ensuremath{\textsc{avg}}\xspace}
\def\AGGone{\ensuremath{\textsc{agg}_1}\xspace}
\def\AGGtwo{\ensuremath{\textsc{agg}_2}\xspace}
\def\MIN{\ensuremath{\textsc{min}}\xspace}
\def\MAX{\ensuremath{\textsc{max}}\xspace}
\def\EVEN{\ensuremath{\textsc{even}}\xspace}
\def\ODD{\ensuremath{\textsc{odd}}\xspace}
\def\rec{\ensuremath{\mathit{rec}}\xspace}
\def\rew{\ensuremath{\mathit{rew}}\xspace}
\def\pre{\ensuremath{\mathit{pre}}\xspace}

\def\ext{\ensuremath{\mathit{ext}}\xspace}
\def\aux{\ensuremath{\mathit{aux}}\xspace}

\def\At{\ensuremath{\mathit{At}}\xspace}
\def\Ag{\ensuremath{\mathit{Ag}}\xspace}
\def\SM{\ensuremath{\mathit{SM}}\xspace}
\def\G{\ensuremath{\mathcal{G}}\xspace}
\def\lparse{\textsc{lparse}\xspace}
\def\sequiv{\ensuremath{\equiv_{\mathit{SE}}}\xspace}
\def\lit{\ensuremath{\mathit{Lit}}\xspace}
\def\wt{\ensuremath{\Sigma}\xspace}
\def\wlit{\ensuremath{\mathit{wLit}}\xspace}
\def\wlita{\ensuremath{\wlit^\ast}\xspace}
\def\wlitp{\ensuremath{\wlit^+}\xspace}
\def\wlitn{\ensuremath{\wlit^-}\xspace}

\def\mon{\ensuremath{\mathit{pos}}\xspace}
\def\gtA{\ensuremath{A_>}\xspace}
\def\ltA{\ensuremath{A_<}\xspace}
\def\AtF{\ensuremath{\mathit{At}^F}\xspace}

\begin{document}

\label{firstpage}

\maketitle

\begin{abstract}
Aggregation functions are widely used in answer set programming for representing and reasoning on knowledge involving sets of objects collectively.
Current implementations simplify the structure of programs in order to optimize the overall performance.
In particular, aggregates are rewritten into simpler forms known as monotone aggregates.
Since the evaluation of normal programs with monotone aggregates is in general on a lower complexity level than the evaluation of normal programs with arbitrary aggregates, any faithful translation function must introduce disjunction in rule heads in some cases.
However, no function of this kind is known.
The paper closes this gap by introducing a polynomial, faithful, and modular translation
for rewriting common aggregation functions into the simpler form accepted by current solvers.
A prototype system allows for experimenting with arbitrary recursive aggregates,
which are  also    supported in the recent version 4.5  of the grounder \textsc{gringo}, using the methods presented in this paper.
\end{abstract}


%\tableofcontents

\section{Introduction}\label{sec:intro}

Answer set programming (ASP) is a declarative language for knowledge representation and reasoning \cite{DBLP:journals/cacm/BrewkaET11}.
In ASP knowledge is encoded by means of logic rules, possibly using disjunction and default negation, interpreted according to the stable model semantics \cite{DBLP:conf/iclp/GelfondL88,DBLP:journals/ngc/GelfondL91}.
Since its first proposal, the basic language was extended by several constructs in order to ease the representation of practical knowledge, and particular interest was given to aggregate functions \cite{DBLP:journals/ai/SimonsNS02,DBLP:journals/ai/LiuPST10,DBLP:conf/aaaiss/BartholomewLM11,DBLP:journals/ai/FaberPL11,DBLP:journals/tocl/Ferraris11,DBLP:journals/tplp/GelfondZ14}.
In fact, aggregates allow for expressing properties on sets of atoms declaratively, and are widely used for example to enforce \emph{functional dependencies}, where % ; 
a rule of the form % following form:
\begin{equation*}
 \bot \leftarrow R'(\overline{X}),\ \COUNT[\overline{Y} : R(\overline{X},\overline{Y},\overline{Z})] \leq 1
\end{equation*}
constrains relation $R$ to satisfy the functional dependency $\overline{X} \rightarrow \overline{Y}$, where
$\overline{X} \cup \overline{Y} \cup \overline{Z}$ is the set of attributes of $R$, and $R'$ is the projection of $R$ on $\overline{X}$.

Among the several semantics proposed for interpreting ASP programs with aggregates, two of them \cite{DBLP:journals/ai/FaberPL11,DBLP:journals/tocl/Ferraris11} are implemented in widely-used ASP solvers \cite{DBLP:journals/tplp/FaberPLDI08,DBLP:journals/ai/GebserKS12}.
The two semantics agree for programs without negated aggregates, and are thus referred indistinctly in this paper as F-stable model semantics.
It is important to observe that the implementation of F-stable model semantics is \emph{incomplete} in current ASP solvers.
In fact, the grounding phase rewrites aggregates into simpler forms known as \emph{monotone} aggregates, and many common reasoning tasks on normal programs with monotone aggregates belong to the first level of the polynomial hierarchy, while in general they belong to the second level for normal programs with aggregates \cite{DBLP:journals/ai/FaberPL11,DBLP:journals/tocl/Ferraris11}.
Since disjunction is not   introduced during the rewriting of aggregates, this is already evidence that currently available rewritings can be correct only if recursion is limited to \emph{convex} aggregates \cite{DBLP:journals/jair/LiuT06}, the largest class of aggregates for which the common reasoning tasks still belong to the first level of the polynomial hierarchy in the normal case \cite{DBLP:conf/lpnmr/AlvianoF13}.

However, non-convex aggregations may arise in several contexts while modeling complex knowledge \cite{DBLP:journals/ai/EiterILST08,DBLP:journals/tplp/EiterFKR12,woltran-aggregates}.
A minimalistic example is provided by the $\Sigma^P_2$-complete problem called \emph{Generalized Subset Sum} \cite{DBLP:journals/jcss/BermanKLPR02}, where two vectors $u$ and $v$ of integers as well as an integer $b$ are given, and the task is to decide whether the formula $\exists x \forall y (ux + vy \neq b)$ is true, where $x$ and $y$ are vectors of binary variables of the same length as $u$ or  $v$, respectively.
For example, for $u = [1,2]$, $v = [2,3]$, and $b = 5$, the task is to decide whether the following formula is true:
$\exists x_1 % \exists 
         x_2 
 \forall y_1 % \forall 
         y_2 (1 \cdot x_1 + 2 \cdot x_2 + 2 \cdot y_1 + 3 \cdot y_2 \neq 5)$.
Any natural encoding of such an instance would include an aggregate of the form $\SUM[1 : x_1, 2 : x_2, 2 : y_1, 3 : y_2] \neq 5$, and it is not immediate how to obtain an equivalent program that comprises monotone aggregates only.

The aim of this paper is to overcome the limitations of current rewritings in order to provide a \emph{polynomial, faithful, and modular} translation % function 
\cite{DBLP:journals/jancl/Janhunen06} that allows to compile logic programs with aggregates into equivalent logic programs that only comprise monotone aggregates.
The paper focuses on common aggregation functions such as \SUM, \AVG, \MIN, \MAX, \COUNT, \EVEN, and \ODD.
Actually, all of them are mapped to possibly non-monotone sums in Section~\ref{sec:tosum}, and non-monotonicity is then eliminated in Section~\ref{sec:rewriting}.
The rewriting is further optimized in Section~\ref{sec:refined} by taking strongly connected components of a refined version of the positive dependency graph into account.
Crucial properties like correctness and modularity are established
in Section~\ref{sec:properties}, followed by the discussion of related work and conclusions.
The proposed rewriting is implemented in a prototype system
(\url{http://alviano.net/software/f-stable-models/}),
and
is   also    adopted in the recent      version 4.5 of the grounder \textsc{gringo}.
With   the prototype, aggregates are represented by reserved predicates, so that the grounding phase can be delegated to \textsc{dlv} \cite{DBLP:conf/datalog/AlvianoFLPPT10} or \textsc{gringo} \cite{DBLP:conf/lpnmr/GebserKKS11}.
The output of a   grounder is then processed to properly encode aggregates for the subsequent stable model search, as performed by \textsc{clasp} \cite{DBLP:journals/ai/GebserKS12}, \textsc{cmodels} \cite{gilima06a}, % \cite{DBLP:conf/lpnmr/LierlerM04},
or \textsc{wasp} \cite{DBLP:journals/tplp/AlvianoDR14}.

\section{Background}\label{sec:background}

Let $\V$ be a set of propositional atoms including $\bot$.
A propositional literal is an atom possibly preceded by one or more occurrences of the \emph{negation as failure} symbol \naf.
An aggregate literal, or simply aggregate, is of one of the following three forms:
\begin{eqnarray}
 \AGGone[w_1 : l_1, \ldots, w_n : l_n] \odot b \quad&
 \COUNT[l_1, \ldots, l_n] \odot b \quad&
 \AGGtwo[l_1, \ldots, l_n] \label{eq:agg}
\end{eqnarray}
where $\AGGone \in \{\SUM,\AVG,\MIN,\MAX\}$, $\AGGtwo \in \{\EVEN,\ODD\}$, $n \geq 0$, $b,w_1,\ldots,w_n$ are integers, $l_1,\ldots,l_n$ are propositional literals, and $\odot \in \{<,\leq,\geq,>,=,\neq\}$.
(Note that $[w_1 : l_1, \ldots, w_n : l_n]$ and $[l_1, \ldots, l_n]$ are multisets.
 This notation of propositional aggregates differs from ASP-Core-2
 (\url{https://www.mat.unical.it/aspcomp2013/ASPStandardization/})
 for ease of presentation.)
A literal is either a propositional literal, or an aggregate.
A rule $r$ is of the following form:
\begin{equation}\label{eq:rule}
 p_1 \vee \cdots \vee p_m \leftarrow l_1 \wedge \cdots \wedge l_n
\end{equation}
where $m \geq 1$, $n \geq 0$, $p_1,\ldots,p_m$ are propositional atoms, and $l_1,\ldots,l_n$ are literals.
The set $\{p_1,\ldots,p_m\}\setminus\{\bot\}$ is referred to as head, denoted by $H(r)$, and the set $\{l_1,\ldots,l_n\}$ is called body, denoted by $B(r)$.
A program $\Pi$ is a finite set of rules.
The set of propositional atoms (different from $\bot$) occurring in a program $\Pi$ is denoted by $\At(\Pi)$, and the set of aggregates occurring in $\Pi$ is denoted by $\Ag(\Pi)$.
%For an atom $p \in \V$, define $\heads(\Pi,p) := \{r \in \Pi \mid p \in H(r)\}$.
%This notation extends to sets of atoms as follows:
%for $S \subseteq \V$, define $\heads(\Pi,S) := \bigcup_{p \in S} \heads(\Pi,p)$.

\begin{example}\label{ex:syntax}
Consider the following program $\Pi_1$:
\begin{equation*} % \[
\begin{array}{@{~}l@{\qquad}l@{\qquad}l@{\qquad}l@{\qquad}l}
 x_1 \leftarrow \naf\naf x_1 & \!x_2 \leftarrow \naf\naf x_2 & \!y_1 \leftarrow \mathit{unequal} & \!y_2 \leftarrow \mathit{unequal} & \!\bot \leftarrow \naf \mathit{unequal} \\
 \multicolumn{5}{@{~}l}{\mathit{unequal} \leftarrow \SUM[1 : x_1, 2 : x_2, 2 : y_1, 3 : y_2] \neq 5}
\end{array}
\end{equation*} % \]
As will be clarified after defining the notion of a stable model, $\Pi_1$ encodes the instance of Generalized Subset Sum introduced in Section~\ref{sec:intro}.
\end{example}

An \emph{interpretation} $I$ is a set of propositional atoms such that $\bot \notin I$.
Relation $\models$ is inductively defined as follows:
\begin{itemize}
\item for $p \in \V$, $I \models p$ if $p \in I$;
\item $I \models \naf l$ if $I \not\models l$;
\item $I \models \SUM[w_1 : l_1, \ldots, w_n : l_n] \odot b$ if $\sum_{i \in [1..n],  I \models l_i} w_i \odot b$;
\item $I \models \AVG[w_1 : l_1, \ldots, w_n : l_n] \odot b$ if $m := |\{i \in [1..n] \mid I \models l_i\}|$, $m \geq 1$, and $\sum_{i \in [1..n],  I \models l_i} \frac{w_i}{m} % w_i / m 
       \odot b$;
\item $I \models \MIN[w_1 : l_1, \ldots, w_n : l_n] \odot b$ if $\min(\{w_i \mid i \in [1..n],  I \models l_i\} \cup \{+\infty\}) \odot b$;
\item $I \models \MAX[w_1 : l_1, \ldots, w_n : l_n] \odot b$ if $\max(\{w_i \mid i \in [1..n], I \models l_i\} \cup \{-\infty\}) \odot b$;
\item $I \models \COUNT[l_1, \ldots, l_n] \odot b$ if $|\{i \in [1..n] \mid I \models l_i\}| \odot b$;
\item $I \models \EVEN[l_1, \ldots, l_n]$ if $|\{i \in [1..n] \mid I \models l_i\}|$ is an even number;
\item $I \models \ODD[l_1, \ldots, l_n]$ if $|\{i \in [1..n] \mid I \models l_i\}|$ is an odd number;
\item for a rule $r$ of the form (\ref{eq:rule}), $I \models B(r)$ if $I \models l_i$ for all $i \in [1..n]$, and $I \models r$ if $H(r) \cap I \neq \emptyset$ when $I \models B(r)$;
\item for a program $\Pi$, $I \models \Pi$ if $I \models r$ for all $r \in \Pi$.
\end{itemize}

% Note that this notation is syntactically and semantically different to the ASP Core 2.0 standard for ease of presentation.
For any expression $\pi$, if $I \models \pi$, we say that
$I$ is a \emph{model} of $\pi$,
$I$ satisfies $\pi$, or % that
$\pi$ is true in~$I$.
In the following, $\top$ will be a shorthand for $\naf \bot$, i.e., $\top$ is a literal true in all interpretations.

\begin{example}\label{ex:models}
Continuing with Example~\ref{ex:syntax}, the models of $\Pi_1$, restricted to the atoms in $\At(\Pi_1)$, are
$X$, $X \cup \{x_1\}$, $X \cup \{x_2\}$, and $X \cup \{x_1,x_2\}$, where $X = \{\mathit{unequal}, y_1, y_2\}$.
\end{example}

The \emph{reduct} of a program $\Pi$ with respect to an interpretation $I$ is obtained by removing rules with false bodies and by fixing the interpretation of all negative literals.
More formally, the following function is inductively defined:
\begin{itemize}
\item for $p \in \V$, $F(I,p) := p$;
\item $F(I,\naf l) := \top$ if $I \not\models l$, % $I \models \naf l$,
      and $F(I,\naf l) := \bot$ otherwise;
\item $F(I,\AGGone[w_1 : l_1, \ldots, w_n : l_n] \odot b) := \AGGone[w_1 : F(I,l_1), \ldots, w_n : F(I,l_n)] \odot b$;
\item $F(I,\COUNT[l_1, \ldots, l_n] \odot b) := \COUNT[F(I,l_1), \ldots, F(I,l_n)] \odot b$;
\item $F(I,\AGGtwo[l_1, \ldots, l_n]) := \AGGtwo[F(I,l_1), \ldots, F(I,l_n)]$;
\item for a rule $r$ of the form (\ref{eq:rule}), $F(I,r) := p_1 \vee \cdots \vee p_m \leftarrow F(I,l_1) \wedge \cdots \wedge F(I,l_n)$;
\item for a program $\Pi$, $F(I,\Pi) := \{F(I,r) \mid r \in \Pi,  I \models B(r)\}$.
\end{itemize}
Program $F(I,\Pi)$ is the reduct of $\Pi$ with respect to $I$.
An interpretation $I$ is a \emph{stable model} of a program $\Pi$ if $I \models \Pi$ and there is no $J \subset I$ such that $J \models F(I,\Pi)$.
Let $\SM(\Pi)$ denote the set of stable models of $\Pi$.
Two programs $\Pi,\Pi'$ are equivalent with respect to a context $V \subseteq \V$, denoted $\Pi \equiv_V \Pi'$, if both $|\SM(\Pi)| = |\SM(\Pi')|$ and $\{I \cap V \mid I \in \SM(\Pi)\} = \{I \cap V \mid I \in \SM(\Pi')\}$.
An aggregate $A$ is \emph{monotone} (in program reducts) if $J \models F(I,A)$ implies $K \models F(I,A)$, for all $J \subseteq K \subseteq I \subseteq \V$, and it is \emph{convex} (in program reducts) if $J \models F(I,A)$ and $L \models F(I,A)$ implies $K \models F(I,A)$, for all $J \subseteq K \subseteq L \subseteq I \subseteq \V$;
when either property applies, $I\models A$ and $J \models F(I,A)$
yield $K \models F(I,A)$,
for all $J \subseteq K \subseteq I$. % \subseteq \V

\begin{example}\label{ex:stable}
Continuing with Example~\ref{ex:models}, the only stable model of $\Pi_1$ is $\{x_1, \mathit{unequal}, y_1, y_2\}$.
Indeed, the reduct $F(\{x_1, \mathit{unequal}, y_1, y_2\},\Pi_1)$ is % as follows:
\begin{equation*} % \[
\begin{array}{l@{\qquad}l@{\qquad}l} % {c}
 x_1 \leftarrow \top & % \qquad
 y_1 \leftarrow \mathit{unequal} & % \qquad 
 y_2 \leftarrow \mathit{unequal}\\
 \multicolumn{3}{l}{\mathit{unequal} \leftarrow \SUM[1 : x_1, 2 : x_2, 2 : y_1, 3 : y_2] \neq 5}
\end{array}
\end{equation*} % \]
and no strict subset of $\{x_1, \mathit{unequal}, y_1, y_2\}$ is a model of the above program.
On the other hand, the reduct $F(\{x_2, \mathit{unequal}, y_1, y_2\},\Pi_1)$ is % as follows:
\begin{equation*} % \[
\begin{array}{l@{\qquad}l@{\qquad}l} % {c}
 x_2 \leftarrow \top & % \qquad
 y_1 \leftarrow \mathit{unequal} & % \qquad 
 y_2 \leftarrow \mathit{unequal}\\
 \multicolumn{3}{l}{\mathit{unequal} \leftarrow \SUM[1 : x_1, 2 : x_2, 2 : y_1, 3 : y_2] \neq 5}
\end{array}
\end{equation*} % \]
and $\{x_2,y_2\}$ is a model of the above program.
Similarly, it can be checked that $\{\mathit{unequal},\linebreak[1] y_1, y_2\}$ and $\{x_1,x_2, \mathit{unequal}, y_1, y_2\}$ are not stable models of $\Pi_1$.
Further note that the aggregate $\SUM[1 : x_1, 2 : x_2, 2 : y_1, 3 : y_2] \neq 5$ % in $\Ag(\Pi)$
is non-convex.
The aggregate is also recursive, or not stratified, a notion that will be formalized later in Section~\ref{sec:refined}.
\end{example}

\section{Compilation}\label{sec:compilation}

Current ASP solvers (as opposed to grounders) only accept a limited set of aggregates, essentially aggregates of the form (\ref{eq:agg}) such that \AGGone is \SUM, $b,w_1,\ldots,w_n$ are non-negative integers, and $\odot$ is $\geq$.
The corresponding class of programs will be referred to as \lparse-like programs.
Hence, compilations from the general language are required.
More formally, what is needed is a polynomial-time computable function associating every program $\Pi$ with an \lparse-like program $\Pi'$ such that $\Pi \equiv_{\At(\Pi)} \Pi'$.
To define such a translation is nontrivial, and indeed most commonly used rewritings that are correct in the stratified case are unsound for recursive aggregates.
% , even if the restriction of being \lparse-like is not enforced in output programs.

\begin{example}\label{ex:wrong}
Consider program $\Pi_1$ from Example~\ref{ex:syntax} and the following program $\Pi_2$, often used as an intermediate step to obtain an \lparse-like program:
\begin{equation*} % \[
\begin{array}{@{~}l@{\qquad}l@{\qquad}l@{\qquad}l@{\qquad}l} % {c}
 x_1 \leftarrow \naf\naf x_1 & % \quad 
 \!x_2 \leftarrow \naf\naf x_2 & % \quad
 \!y_1 \leftarrow \mathit{unequal} & % \quad 
 \!y_2 \leftarrow \mathit{unequal} & % \quad
 \!\bot \leftarrow \naf \mathit{unequal} \\
 \multicolumn{5}{@{~}l}{\mathit{unequal} \leftarrow \SUM[1 : x_1, 2 : x_2, 2 : y_1, 3 : y_2] > 5} \\
 \multicolumn{5}{@{~}l}{\mathit{unequal} \leftarrow \SUM[1 : x_1, 2 : x_2, 2 : y_1, 3 : y_2] < 5}
\end{array}
\end{equation*} % \]
The two programs only minimally differ:
the last rule of $\Pi_1$ is replaced by two rules in~$\Pi_2$, following the intuition that the original aggregate is true in an interpretation $I$ if and only if either $I \models \SUM[1 : x_1, 2 : x_2, 2 : y_1, 3 : y_2] > 5$ or $I \models \SUM[1 : x_1, 2 : x_2, 2 : y_1,\linebreak[1] 3 : y_2] < 5$.
However, the two programs are not equivalent.
Indeed, it can be checked that $\Pi_2$ has no stable model, and in particular $\{x_1, \mathit{unequal}, y_1, y_2\}$ is not stable because $F(\{x_1, \mathit{unequal}, y_1, y_2\},\Pi_2)$ is % as follows:
\begin{equation*} % \[
\begin{array}{l@{\qquad}l@{\qquad}l} % {c}
 x_1 \leftarrow \top & % \qquad
 y_1 \leftarrow \mathit{unequal} & % \qquad 
 y_2 \leftarrow \mathit{unequal} \\
 \multicolumn{3}{l}{\mathit{unequal} \leftarrow \SUM[1 : x_1, 2 : x_2, 2 : y_1, 3 : y_2] > 5}
\end{array}
\end{equation*} % \]
and $\{x_1\}$ is one of its models.
\end{example}

Also replacing negative integers may change the semantics of % some
programs.

\begin{example}\label{ex:wrong:2}
Let $\Pi_3 := \{p \leftarrow \SUM[1 : p, -1 : q] \geq 0,\ p \leftarrow q,\ q \leftarrow p\}$.
Its only stable model is $\{p,q\}$.
The negative integer is usually removed by means of a rewriting adapted from pseudo-Boolean constraint solvers, which replaces each element $w : l$ in (\ref{eq:agg}) such that $w < 0$ by % with 
$-w : \naf l$, and also adds $-w$ to $b$.
The resulting program in the example is $\{p \leftarrow \SUM[1 : p,\linebreak[1] 1 : \naf q] \geq 1,\linebreak[1] p \leftarrow q,\ q \leftarrow p\}$, which has no stable models.
\end{example}

Actually, stable models cannot be preserved in general by rewritings such as those hinted in the above examples unless the polynomial hierarchy collapses to its first level.
In fact, while checking the existence of a stable model is $\Sigma^P_2$-complete for programs with atomic heads, this problem is in NP for \lparse-like programs with atomic heads, and disjunction is necessary for modeling $\Sigma^P_2$-hard instances.
It follows that, in order to be correct, a polynomial-time compilation must possibly introduce disjunction when rewriting recursive programs.
This intuition is formalized in Section~\ref{sec:rewriting}.
Before, in Section~\ref{sec:tosum}, the structure of input programs is simplified by mapping all aggregates to conjunctions of sums, where comparison operators are either
\ifInput
$>$ or $\neq$.
While $>$ can be viewed as $\geq$ relative to an incremented bound $b+1$,
negative integers as well as $\neq$ constitute the remaining gap to \lparse-like programs.
\else
$\geq$ or $\neq$.
\fi

\subsection{Mapping to sums}\label{sec:tosum}

\ifInput
% mainfile: main.tex

The notion of strong equivalence \cite{lipeva01a,turner03a,DBLP:journals/tocl/Ferraris11} will be used in this section.
Let $\pi := l_1 \wedge \cdots \wedge l_n$ be a conjunction of literals, for some $n \geq 0$.
A pair $(J,I)$ of interpretations such that $J\subseteq I$ is an \emph{SE-model} of $\pi$ if $I \models \pi$ and $J \models F(I,l_1) \wedge \cdots \wedge F(I,l_n)$.
Two conjunctions $\pi,\pi'$ are \emph{strongly equivalent}, denoted by $\pi \sequiv \pi'$, if they have the same SE-models.
Strong equivalence means that replacing $\pi$ by $\pi'$ preserves the stable models of any logic program.

\begin{proposition}[\citeauthoryear{Lifschitz, Pearce, and Valverde}{Lifschitz
  et~al\mbox{.}}{2001}; \citeauthoryear{Turner}{Turner}{2003}; \citeauthoryear{Ferraris}{Ferraris}{2011}]\label{prop:se}
Let $\pi,\pi'$ be two conjunctions of literals such that $\pi \sequiv \pi'$.
Let $\Pi$ be a program, and $\Pi'$ be the program obtained from $\Pi$ by replacing any occurrence of $\pi$ by $\pi'$.
Then, it holds that $\Pi \equiv_\V \Pi'$ (where $\V$ is the set of all propositional atoms).
\end{proposition}

The following strong equivalences can be proven by showing equivalence with respect to models, and by noting that \naf is neither introduced nor eliminated:
\begin{enumerate}[leftmargin=*,labelsep=2pt,label=(\Alph{enumi})]%\renewcommand{\labelenumi}{(\Alph{enumi})}
% \item[(SE1)] $\SUM[w_1 : l_1, \ldots, w_n : l_n] \odot b \sequiv \SUM[w : l_i \mid i \in [1..n],\ \forall j \in [1..i-1]\ l_j \neq l_i,\ w := \sum_{k \in [1..n],\ l_k = l_i} w_k,\ w \neq 0] \odot b$;
\item % \item[\textnormal{(S1)}] 
  $\SUM[w_1 : l_1, \ldots, w_n : l_n] < b \sequiv \SUM[-w_1 : l_1, \ldots, -w_n : l_n] > -b$%;
\item % \item[\textnormal{(S2)}]
  $\SUM[w_1 : l_1, \ldots, w_n : l_n] \leq b \sequiv \SUM[-w_1 : l_1, \ldots, -w_n : l_n] > -b-1$%;
\item % \item[\textnormal{(S3)}]
  $\SUM[w_1 : l_1, \ldots, w_n : l_n] \geq b \equiv_{SE} \SUM[w_1 : l_1, \ldots, w_n : l_n] > b-1$%;
\item % \item[\textnormal{(S4)}]
  $\begin{array}[t]{@{}l@{}l@{}}
   \SUM[w_1 : l_1, \ldots, w_n : l_n] = b \sequiv {} &             
   \SUM[w_1 : l_1, \ldots, w_n : l_n] > b-1 \wedge {} \\ & 
   \SUM[-w_1 : l_1, \ldots, -w_n : l_n] > -b-1%\text{.}
   \end{array}$
\end{enumerate}
For instance, given an interpretation~$I$, (A) is based on the fact that
$\sum_{i \in [1..n], I \models l_i} w_i < b$
if and only if
$\sum_{i \in [1..n], I \models l_i} -w_i > -b$, so that
$I \models \SUM[w_1 : l_1, \ldots, w_n : l_n] < b$ if and only if
$I \models \SUM[-w_1 : l_1, \ldots, -w_n : l_n] > -b$.
Similar observations apply to (B)--(D),
and strong equivalences as follows hold for further aggregates:
\begin{enumerate}[leftmargin=*,labelsep=2pt,label=(\Alph{enumi})]\setcounter{enumi}{4}%\renewcommand{\labelenumi}{(\Alph{enumi})}
\item % \item[\rlap{\textnormal{(S6)}}\phantom{\textnormal{(S15)}}]
  $\begin{array}[t]{@{}l@{}l@{}}
   \AVG[w_1 : l_1, \ldots, w_n : l_n] \odot b \sequiv {} & 
   \SUM[w_1-b : l_1, \ldots, w_n-b : l_n] \odot 0 \wedge {} \\ & 
   \SUM[1:l_1, \ldots, 1:l_n] > 0%\text{;}
   \end{array}$
\item % \item[\rlap{\textnormal{(S7)}}\phantom{\textnormal{(S15)}}]
  $\MIN[w_1 : l_1, \ldots, w_n : l_n] < b \sequiv \SUM[1:l_i \mid i \in [1..n], w_i < b] > 0$%;
\item % \item[\rlap{\textnormal{(S8)}}\phantom{\textnormal{(S15)}}]
  $\MIN[w_1 : l_1, \ldots, w_n : l_n] \leq b \sequiv \SUM[1:l_i \mid i \in [1..n], w_i \leq b] > 0$%;
\item % \item[\rlap{\textnormal{(S9)}}\phantom{\textnormal{(S15)}}]
  $\MIN[w_1 : l_1, \ldots, w_n : l_n] \geq b \sequiv \SUM[-1:l_i \mid i \in [1..n], w_i < b] > -1$%;
\item % \item[\rlap{\textnormal{(S10)}}\phantom{\textnormal{(S15)}}]
  $\MIN[w_1 : l_1, \ldots, w_n : l_n] > b \sequiv \SUM[-1:l_i \mid i \in [1..n], w_i \leq b] > -1$%;
\item % \item[\rlap{\textnormal{(S11)}}\phantom{\textnormal{(S15)}}]
  $\MIN[w_1 : l_1, \ldots, w_n : l_n] = b \sequiv \SUM[1-n\cdot(b-w_i) : l_i \mid i \in [1..n], w_i \leq b] > 0$%;
\item % \item[\rlap{\textnormal{(S12)}}\phantom{\textnormal{(S15)}}]
  $\MIN[w_1 : l_1, \ldots, w_n : l_n] \neq b \sequiv \SUM[n\cdot(b-w_i)-1 : l_i \mid i \in [1..n], w_i \leq\nolinebreak b] >\nolinebreak -1$%;
\item % \item[\rlap{\textnormal{(S13)}}\phantom{\textnormal{(S15)}}]
  $\MAX[w_1 : l_1, \ldots, w_n : l_n] \odot b \sequiv \MIN[-w_1 : l_1, \ldots, -w_n : l_n] \ f(\odot)\ {-b}$\\ % ,
  where $<{} \stackrel{f}{\mapsto} {}>$, $\leq{} \stackrel{f}{\mapsto} {}\geq$, $\geq{} \stackrel{f}{\mapsto} {}\leq$, $>{} \stackrel{f}{\mapsto} {}<$, $={} \stackrel{f}{\mapsto} {}=$, and $\neq{} \stackrel{f}{\mapsto} {}\neq$%;
\item % \item[\rlap{\textnormal{(S5)}}\phantom{\textnormal{(S15)}}]
  $\COUNT[l_1, \ldots, l_n] \odot b \sequiv \SUM[1 : l_1, \ldots, 1 : l_n] \odot b$%;
\item % \item[\rlap{\textnormal{(S14)}}\phantom{\textnormal{(S15)}}]
  $\EVEN[l_1, \ldots, l_n] \sequiv \bigwedge_{i \in [1..\left\lceil n/2 \right\rceil]} \SUM[1:l_1, \ldots, 1:l_n] \neq 2 \cdot i - 1$%;
\item % \item[\textnormal{(S15)}]
  $\ODD[l_1, \ldots, l_n] \sequiv \bigwedge_{i \in [0..\left\lfloor n/2 \right\rfloor]} \SUM[1:l_1, \ldots, 1:l_n] \neq 2 \cdot i$%.
\end{enumerate}
Given a program $\Pi$, the successive application of (A)--(O),
from the last to the first, gives an equivalent program $\Pi'$
whose aggregates are sums with comparison operators $>$ and $\neq$.

\begin{example}
Let $\Pi_4 := \{p \vee q \leftarrow,\ p \leftarrow \AVG[5 : p, 3 : p, 2 : q, 7 : q] \geq 4\}$.
By applying~(E), the aggregate becomes
$\SUM[1 : p, -1 : p, -2 : q, 3 : q] \geq 0 \wedge 
 \SUM[1 : p,  1 : p,  1 :\nolinebreak q, 1 :\nolinebreak q] >\nolinebreak 0$,
and an application of (C) yields
$\SUM[1 : p, -1 : p, -2 : q, 3 : q] > -1 \wedge\linebreak[1]
 \SUM[1 : p,  1 : p,\linebreak[1] 1 : q,\linebreak[1] 1 : q] > 0$.
Simplifying the latter expression leads to the program
$\Pi_4' := \{p \vee q \leftarrow,\linebreak[1] p \leftarrow \SUM[1 : q] > -1 \wedge \SUM[2 : p, 2 : q] > 0\}$.
Note that $\{p\}$ is the unique stable model of both $\Pi_4$ and $\Pi_4'$,
so that $\Pi_4\equiv_{\{p,q\}}\Pi_4'$.
\end{example}

\subsection{Eliminating non-monotone aggregates}\label{sec:rewriting}

The structure of input programs can be further simplified by eliminating non-monotone aggregates.
%Actually, this is the main result in the paper.
Without loss of generality, we hereinafter assume aggregates to be of the form
\begin{equation}\label{eq:aggregate-normal}
  \SUM[w_1 : l_1, \ldots, w_n : l_n] \odot b
\end{equation}
such that $\odot \in \{>,\neq\}$.
For $A$ of the form~(\ref{eq:aggregate-normal}),
by $\lit(A):=\{l_1,\ldots,l_n\}\setminus\{\bot\}$,
we refer to the set of propositional literals (different from $\bot$) occurring in~$A$.
Moreover, let $\wt(l,A):=\sum_{i \in [1..n], l_i=l}w_i$
denote the weight of any $l\in\lit(A)$.
We write
$\wlita(A) := [\wt(l,A):l \mid\linebreak[1] l\in\lit(A),\linebreak[1] \wt(l,A)\neq 0]$,
$\wlitp(A) := [\wt(l,A):l \mid\linebreak[1] l\in\lit(A),\linebreak[1] \wt(l,A)> 0]$, and
$\wlitn(A) := [\wt(l,A):l \mid\linebreak[1] l\in\lit(A),\linebreak[1] \wt(l,A)< 0]$
to distinguish the (multi)sets of literals
associated with non-zero, positive, or negative weights, respectively, in~$A$.
For instance, letting $A:=\SUM[1 : p, -1 : p, -2 : q, 3 : q] > -1$,
we have that $\wlita(A)=\wlitp(A)=[1:q]$ and $\wlitn(A)=[]$.
In the following, we call an aggregate~$A$ of the form (\ref{eq:aggregate-normal})
non-monotone if $\{p\in \V \mid (w:p)\in\wlitn(A)\}\neq\emptyset$, or if $\odot$ is $\neq$,
thus disregarding special cases in which $A$ would still be monotone or convex.
(The rewritings presented below are correct also in such cases,
but they do not exploit the particular structure of an aggregate for avoiding
the use of disjunction in rule heads.)
% Abusing of terminology, hereinafter an aggregate $A$ of the form (\ref{eq:aggregate-normal}) will be considered non-monotone if $\wlitn(A) \cap \V \neq \emptyset$, or $\odot$ is $\neq$, even in the few special cases in which $A$ would still be monotone or convex.
% (The rewritings presented in this paper are correct also in such cases, for which however optimizations are possible in order to not introduce disjunction in rule heads.)

For an aggregate~$A$ of the form~(\ref{eq:aggregate-normal})
such that $\odot$ is $>$ and a set $V\subseteq\V$ of atoms,
we define a rule with a fresh propositional atom $\aux$ as head and a
monotone aggregate as body by:
\begin{equation}\label{eq:aggregate:rewrite}
  \aux \leftarrow \SUM\left(
                       \begin{array}{@{}l@{}}
                       \wlitp(A) 
                       \cup {} \\{}
                       [-w:p^F \mid (w:p)\in\wlitn(A),p\in V]
                       \cup {} \\{}
                       [-w :\naf l \mid (w:l)\in\wlitn(A),l\notin V]
                       \end{array}
                      \right)
                  > b - \sum_{(w:l)\in\wlitn(A)}w
\end{equation}
Note that (\ref{eq:aggregate:rewrite}) introduces a fresh,
hidden propositional atom $p^F$ \cite{eitowo05a,jannie12a}
for any $p\in V$ associated with a negative weight in~$A$.
However, when $V=\emptyset$, every $(w:l)\in\wlitn(A)$
is replaced by $-w:\naf l$,
thus rewarding the falsity of~$l$ rather than penalizing~$l$,
which is in turn compensated by adding $-w$ to the bound~$b$;
such a replacement % of literals with negative weights
preserves models \cite{DBLP:journals/ai/SimonsNS02},
but in general not stable models \cite{DBLP:journals/tplp/FerrarisL05}.
By $\mon(A,V)$, we denote the program including rule~(\ref{eq:aggregate:rewrite}) along with the
following rules for every $p\in V$ such that $(w:\nolinebreak p)\in\wlitn(A)$:
\begin{align}
       p^F & {} \leftarrow \naf p \label{eq:pF:1} \\
       p^F & {} \leftarrow \aux \label{eq:pF:2} \\
p \vee p^F & {} \leftarrow \naf\naf \aux \label{eq:pF:3}
\end{align}
Intuitively, any atom $p^F$ introduced in $\mon(A,V)$ must be true whenever $p$ is false,
but also when $\aux$ is true, so to implement the concept of \emph{saturation}
\cite{eitgot95a}.
Rules (\ref{eq:pF:1}) and (\ref{eq:pF:2}) encode such an intuition.
Moreover, rule (\ref{eq:pF:3}) guarantees that at least one of $p$ and $p^F$
belongs to any model of reducts obtained from interpretations~$I$ containing $\aux$.
In fact, $p^F$ represents the falsity of~$p$
in the reduct of rule~(\ref{eq:aggregate:rewrite})
with respect to~$I$ in order to test the satisfaction of the monotone aggregate in~(\ref{eq:aggregate:rewrite})
relative to subsets of~$I$.
For a program~$\Pi$,
the rewriting $\rew(\Pi,A,V)$ is the union of $\mon(A,V)$ and the program
obtained from~$\Pi$ by replacing any occurrence of~$A$ by $\aux$. % \footnote{%
% For simplicity, we omit an explicit reference to~$\Pi$ in $\rew(A,V)$,
% since $\Pi$ will always be clear from the context.}
That is, $\rew(\Pi,A,V)$ eliminates a (possibly) non-monotone aggregate~$A$
with comparison operator $>$ in favor of a monotone
aggregate and disjunction within the subprogram $\mon(A,V)$.
In this section, we further rely on $V=\nolinebreak\V$,
i.e., saturation is applied to all atoms associated with negative weights in~$A$,
while a refinement based on positive dependencies will be 
provided in the next section.

\begin{example}\label{ex:gt}
Consider~$\Pi_3$ from Example~\ref{ex:wrong:2}
whose first rule is strongly equivalent to $p \leftarrow \SUM[1 : p,\linebreak[1] -1 : q] > -1$.
For $A:= \SUM[1 : p,-1 : q] > -1$, the program $\mon(A,\V)$ is as follows:
\begin{equation*}
\begin{array}{l@{\qquad}l@{\qquad}l@{\qquad}l}
\aux      \leftarrow \SUM[1 : p, 1 : q^F] > 0 &
 q^F      \leftarrow \naf q & 
 q^F      \leftarrow \aux &
q\vee q^F \leftarrow \naf\naf \aux
\end{array}
\end{equation*}
Moreover, we obtain 
$\rew(\Pi_3,A,\V) = \mon(A,\V) \cup \{p\leftarrow \aux,\ p \leftarrow q,\ q \leftarrow p\}$
as the full rewriting of~$\Pi_3$ for~$A$ and~$\V$.
One can check that no strict subset of $\{p,q,\aux,q^F\}$ is a model of $\rew(\Pi_3,A,\V)$ or
the reduct $F(\{p,q,\aux,q^F\},\rew(\Pi_3,A,\V))$, respectively,
where the latter includes $q\vee q^F \leftarrow\top$.
In fact, $\SM(\rew(\Pi_3,A,\V))=\{\{p,q,\aux,q^F\}\}$ and
$\SM(\Pi_3)=\{\{p,q\}\}$
yield that $\Pi_3 \equiv_{\{p,q\}} \rew(\Pi_3,A,\V)$.
\end{example}

We further extend the rewriting to an aggregate
$A := \SUM[w_1 : l_1, \ldots, w_n : l_n] \neq b$ by considering two cases
based on splitting~$A$ into
$\gtA := \SUM[w_1 : l_1, \ldots, w_n : l_n] > b$ and
$\ltA := \SUM[-w_1 : l_1, \ldots, -w_n : l_n] > -b$.
While $\gtA$ is true in any interpretation~$I$ such that
$\sum_{i \in [1..n], I \models l_i} w_i > b$,
in view of the strong equivalence given in~(A),
$I$ satisfies $\ltA$ if and only if
$\sum_{i \in [1..n], I \models l_i} w_i < b$.
For a program~$\Pi$ and $V\subseteq\V$,
we let $\mon(A,V):=\mon(\gtA,V)\cup\mon(\ltA,V)$, and
the rewriting
$\rew(\Pi,A,V)$ is the union of
$\mon(A,V)$ and the program
obtained from~$\Pi$ by replacing any occurrence of~$A$ by $\aux$,
where the fresh propositional atom $\aux$ serves as the head of rules of the form~(\ref{eq:aggregate:rewrite})
in both $\mon(\gtA,V)$ and $\mon(\ltA,V)$.
Note that $\mon(A,V)$ % $\rew(\Pi,A,V)$ 
also introduces fresh propositional atoms~$p^F$
% associated with the falsity of~$p$
for any $p\in V$ such that $(w:p)\in\wlita(A)$.
Again, an atom~$p^F$ represents the falsity of~$p$
in the reduct of rule~(\ref{eq:aggregate:rewrite}) from 
either $\mon(\gtA,V)$ or $\mon(\ltA,V)$ with respect to
interpretations~$I$ containing $\aux$,
which allows for testing the satisfaction of monotone counterparts of
$\gtA$ and $\ltA$ relative to subsets of~$I$.

\begin{example}\label{ex:rew}
Consider program $\Pi_1$ from Example~\ref{ex:syntax},
and let
$A:=\SUM[1 :\nolinebreak x_1,\linebreak[1] 2 : x_2, 2 : y_1, 3 : y_2] \neq 5$.
Then, we obtain the following rewriting $\rew(\Pi_1,A,\V)$:
\begin{equation*} % \[
\begin{array}{llll}
 x_1 \leftarrow \naf\naf x_1 & x_2 \leftarrow \naf\naf x_2 & y_1 \leftarrow \mathit{unequal} & y_2 \leftarrow \mathit{unequal} \\
 \bot \leftarrow \naf \mathit{unequal} & \multicolumn{3}{l}{\aux \leftarrow \SUM[1 : x_1, 2 : x_2, 2 : y_1, 3 : y_2] > 5} \\ 
 \mathit{unequal} \leftarrow \aux & \multicolumn{3}{l}{\aux \leftarrow \SUM[1 : x_1^F, 2 : x_2^F, 2 : y_1^F, 3 : y_2^F] > 3} \\ 
 \phantom{x_1\vee{}}x_1^F \leftarrow \naf x_1 & \phantom{x_2\vee{}}x_2^F \leftarrow \naf x_2 & \phantom{y_1\vee{}}y_1^F \leftarrow \naf y_1 & \phantom{y_2\vee{}}y_2^F \leftarrow \naf y_2 \\
 \phantom{x_1\vee{}}x_1^F \leftarrow \aux     & \phantom{x_2\vee{}}x_2^F \leftarrow \aux     & \phantom{y_1\vee{}}y_1^F \leftarrow \aux     & \phantom{y_2\vee{}}y_2^F \leftarrow \aux \\
 x_1 \vee x_1^F \leftarrow \naf\naf \aux & x_2 \vee x_2^F \leftarrow \naf\naf \aux &  y_1 \vee y_1^F \leftarrow \naf\naf \aux & y_2 \vee y_2^F \leftarrow \naf\naf \aux
\end{array}
\end{equation*} % \]
The only stable model of $\rew(\Pi_1,A,\V)$ is
$\{x_1, \mathit{unequal}, y_1, y_2, \aux, x_1^F, x_2^F, y_1^F, y_2^F\}$.
In particular,
note that $x_1\leftarrow\top$, $x_2^F\leftarrow\top$,
$y_1 \vee y_1^F \leftarrow\top$, and $y_2 \vee y_2^F \leftarrow\top$
belong to the reduct,
and any choice between $y_1$ and $y_1^F$ as well as $y_2$ and $y_2^F$
leads to the satisfaction of
$\SUM[1 : x_1, 2 : x_2, 2 : y_1, 3 : y_2] > 5$ or
$\SUM[1 : x_1^F, 2 : x_2^F, 2 : y_1^F, 3 : y_2^F] > 3$ along with saturation.
As a consequence, % we have that
$\Pi_1\equiv_{\{x_1,x_2,\mathit{unequal},y_1,y_2\}}\rew(\Pi_1,A,\V)$.
\end{example}

The subprogram $\mon(A,V)$
for $A$ of the form~(\ref{eq:aggregate-normal}) such that $\odot \in \{>,\neq\}$
and $V\subseteq\nolinebreak\V$ is \lparse-like.
Moreover, the rewriting $\rew(\Pi,A,V)$ can be iterated to eliminate all
non-monotone aggregates~$A$ from~$\Pi$.
Thereby, it is important to note that fresh propositional atoms introduced in $\mon(A_1,V_1)$ and $\mon(A_2,V_2)$
for $A_1\neq A_2$ are distinct.
As hinted in the above examples,
$\rew(\Pi,A,\V)$ preserves stable models of~$\Pi$,
which extends to an iterated elimination of aggregates.
Before formalizing respective properties in Section~\ref{sec:properties}, however,
we refine $\rew(\Pi,A,V)$ to subsets~$V$ of~$\V$
based on positive dependencies % and recursion 
in~$\Pi$.

\subsection{Refined rewriting}\label{sec:refined}

Given a program~$\Pi$ such that all aggregates in~$\Ag(\Pi)$ are of the form~(\ref{eq:aggregate-normal})
for $\odot \in \{>,\neq\}$,
the (positive) \emph{dependency graph} $\G_\Pi$ of~$\Pi$
consists of the vertices $\At(\Pi) \cup \Ag(\Pi)$
and (directed) arcs $(\alpha,\beta)$ if either of the following conditions
holds for $\alpha,\beta\in\At(\Pi) \cup \Ag(\Pi)$:%
\begin{itemize}
\item there is a rule $r \in \Pi$ such that $\alpha \in H(r)$ and $\beta \in B(r)$;
\item $\alpha \in \Ag(\Pi)$ is of the form (\ref{eq:aggregate-normal}) such that $\odot$ is $>$ and $(w:\beta)\in\wlitp(\alpha)$;
\item $\alpha \in \Ag(\Pi)$ is of the form (\ref{eq:aggregate-normal}) such that $\odot$ is $\neq$ and $(w:\beta)\in\wlita(\alpha)$.
\end{itemize}
That is, $\G_\Pi$ includes arcs from atoms in~$H(r)$ to positive literals in $B(r)$
for rules $r\in\nolinebreak\Pi$, and from aggregates $A\in\Ag(\Pi)$ to atoms associated with a
positive or non-zero weight in~$A$ if the comparison operator of~$A$ is $>$ or $\neq$,
respectively.
A strongly connected component of $\G_\Pi$, also referred to as \emph{component} of $\Pi$, is a maximal subset~$C$ of $\At(\Pi) \cup \Ag(\Pi)$ such that any $\alpha\in C$
reaches each $\beta\in C$ via a path in $\G_\Pi$.
The set of propositional atoms in the component of~$\Pi$ containing an aggregate $A \in \Ag(\Pi)$ is denoted by $\rec(\Pi,A)$ (or $\rec(\Pi,A):=\emptyset$ when $A\notin\Ag(\Pi)$).
Then,
the rewriting $\rew(\Pi,A,\rec(\Pi,A))$ 
restricts saturation for fresh propositional atoms~$p^F$ introduced in
$\mon(A,\rec(\Pi,A))$ to atoms $p\in\rec(\Pi,A)$ occurring in~$A$.

\begin{figure}
    \figrule
    \centering
    \begin{minipage}{.35\textwidth}
        \tikzstyle{node} = [rectangle, draw, rounded corners, text centered]
        \tikzstyle{line} = [draw, ->, thick]
        \tikzstyle{dline} = [draw, ->, thick, dashed]
        \tikzstyle{branch} = [draw, rectangle, rotate=45, anchor=center]

        \begin{tikzpicture}
            \node at (0,0) (p) {$p$};
            \node at (2,0) (q) {$q$};
            \node at (0,-1) (A) {$\SUM[1 : p, -1 : q] > -1$};
            
            \path[line] (p) edge[bend right] (A);
            \path[line] (A) edge[bend right] (p);
            \path[line] (p) edge (q);
            \path[dline] (q) edge[bend right] (p);
        \end{tikzpicture}
    \end{minipage}
    \quad
    \begin{minipage}{.525\textwidth}
        \begin{equation*}
            \Pi_3 = \left\{
                \begin{array}{l}
                    p \leftarrow \SUM[1 : p, -1 : q] > -1\\
                    p \leftarrow q\\
                    q \leftarrow p
                \end{array}
            \right.
            \hspace{-2em}
            \begin{array}{l}
                \left.
                    \begin{array}{l}
                        \\
                        \\
                    \end{array}
                \right\} = \Pi_3'\\
                \\
            \end{array}
        \end{equation*}
    \end{minipage}
    \caption{Dependency graphs considered in Example~\ref{ex:positive}: the dashed arc belongs to $\G_{\Pi_3}$, but not to $\G_{\Pi_3'}$.}\label{fig:dep-graph}
    \figrule
\end{figure}
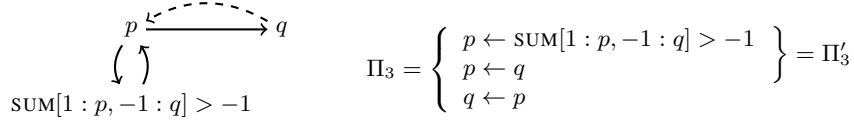
\begin{example}\label{ex:positive}
The dependency graph of  program $\Pi_3$ from Example~\ref{ex:wrong:2} is shown in Fig.~\ref{fig:dep-graph}, where the first rule of $\Pi_3$ is identified
with $p \leftarrow \SUM[1 : p, -1 : q] > -1$.
Let $A$ denote the aggregate $\SUM[1 : p, -1 : q] > -1$.
First of all, note that there is no arc connecting $A$ to $q$ because $(w:q) \notin \wlitp(A)$.
However, $A$ reaches $q$ in $\G_{\Pi_3}$ via $p$, and since also $q$ reaches $A$ via~$p$, we have that $\rec(\Pi_3,A)=\{p,q\}$, and thus $\rew(\Pi_3,A,\rec(\Pi_3,A))=\rew(\Pi_3,A,\V)$.

Now consider $\Pi_3':=\Pi_3\setminus\{q \leftarrow p\}$, whose dependency graph is obtained by removing arc $(q,p)$ from $\G_{\Pi_3}$, i.e., the dashed arc in Fig.~\ref{fig:dep-graph}.
Note that $q$ does not reach $A$ in $\G_{\Pi_3'}$, and therefore $\rec(\Pi_3',A)=\{p\}$.
In this case, $\rew(\Pi_3',A,\rec(\Pi_3',A))=\{\aux\leftarrow \SUM[1 :\nolinebreak p,\linebreak[1] 1 : \naf q] >\nolinebreak 0,\linebreak[1] p \leftarrow\nolinebreak \aux,$ $p\leftarrow q\}$,
where $\SM(\rew(\Pi_3',A,\rec(\Pi_3',A)))=\{\{p,\aux\}\}$ and
$\SM(\Pi_3')=\{\{p\}\}$ yield that
$\Pi_3'\equiv_{\{p,q\}}\rew(\Pi_3',A,\rec(\Pi_3',A))$.
\end{example}

\begin{example}\label{ex:component}
Program $\Pi_1$ from Example~\ref{ex:syntax} has the components $\{x_1\}$, $\{x_2\}$, and 
$\{\mathit{unequal},y_1,y_2,A\}$ for $A := \SUM[1 : x_1; 2 : x_2; 2 : y_1; 3 : y_2] \neq 5$.
Thus, $\rew(\Pi_1,A,\rec(\Pi_1,A))$ comprises the following rules:
\begin{equation*} % \[
\begin{array}{lll@{\qquad}l}
 x_1 \leftarrow \naf\naf x_1 & x_2 \leftarrow \naf\naf x_2 & y_1 \leftarrow \mathit{unequal} & y_2 \leftarrow \mathit{unequal} \\
 \bot \leftarrow \naf \mathit{unequal} & \multicolumn{3}{l}{\aux \leftarrow \SUM[1 : x_1, 2 : x_2, 2 : y_1, 3 : y_2] > 5} \\ 
 \mathit{unequal} \leftarrow \aux & \multicolumn{3}{l}{\aux \leftarrow \SUM[1 : \naf x_1, 2 : \naf x_2, 2 : y_1^F, 3 : y_2^F] > 3} \\ 
 \phantom{y_1\vee{}}y_1^F \leftarrow \naf y_1 & \phantom{y_2\vee{}}y_2^F \leftarrow \naf y_2 \\
 \phantom{y_1\vee{}}y_1^F \leftarrow \aux     & \phantom{y_2\vee{}}y_2^F \leftarrow \aux \\
 y_1 \vee y_1^F \leftarrow \naf\naf \aux & y_2 \vee y_2^F \leftarrow \naf\naf \aux
\end{array}
\end{equation*} % \]
In contrast to $\rew(\Pi_1,A,\V)$ in Example~\ref{ex:rew},
$x_1$ and~$x_2$ are mapped to $\naf x_1$ and $\naf x_2$,
rather than $x_1^F$ and $x_2^F$, in the rule
$\aux \leftarrow \SUM[1 : \naf x_1, 2 : \naf x_2, 2 : y_1^F, 3 : y_2^F] > 3$
from $\mon(\SUM[-1 : x_1, -2 : x_2, -2 : y_1, -3 : y_2] > -5,\rec(\Pi_1,A))$.
Hence, the reduct of $\rew(\Pi_1,A,\rec(\Pi_1,A))$
with respect to $\{x_1, \mathit{unequal}, y_1, y_2, \aux, y_1^F, y_2^F\}$
includes $\aux \leftarrow \SUM[1 : \bot, 2 : \top, 2 : y_1^F, 3 : y_2^F] > 3$
as well as $x_1\leftarrow\top$,
$y_1 \vee y_1^F \leftarrow\top$, and $y_2 \vee y_2^F \leftarrow\top$.
As a consequence, any model containing $y_1^F$ or $y_2^F$ entails $\aux$,
and $\aux \leftarrow \SUM[1 : x_1, 2 : x_2, 2 : y_1, 3 : y_2] > 5$
yields $\aux$ when $y_1$ and $y_2$ are both true.
In fact, $\{x_1, \mathit{unequal}, y_1, y_2, \aux, y_1^F, y_2^F\}$
is the only stable model of $\rew(\Pi_1,A,\rec(\Pi_1,A))$,
so that $\Pi_1\equiv_{\{x_1,x_2,\mathit{unequal},y_1,y_2\}}\rew(\Pi_1,A,\rec(\Pi_1,A))$.
\end{example}

The above examples illustrate that saturation can be restricted to
atoms~$p$ sharing the same component of~$\Pi$ with an aggregate~$A$,
where a fresh propositional atom~$p^F$ is introduced in 
$\mon(A,\rec(\Pi,A))$ when $p$ has a
negative or non-zero weight in~$A$,
depending on whether the comparison operator of~$A$ is $>$ or $\neq$, respectively.
That is, the refined rewriting uses disjunction only if $A$ is
a recursive non-monotone aggregate.
\begin{comment}
\footnote{%
In the special case that $\rec(\Pi,A)=\{p\}$ for
$A$ of the form~(\ref{eq:aggregate-normal}) such that $\odot$ is $\neq$,
% $A:=\SUM[w_1 : l_1, \ldots, w_n : l_n] \neq b$,
introducing $p^F$ could be avoided by augmenting $\mon(A,\emptyset)$ with
$\aux\leftarrow \SUM[ s\cdot w : l \mid (w:l)\in\wlita(A), l\neq\nolinebreak p] >\nolinebreak s\cdot b \wedge
                \SUM[-s\cdot w : l \mid (w:l)\in\wlita(A), l\neq p] > -s\cdot (b+\wt(p,A))$
in order to handle any interpretation~$I$ such that
$s\cdot b < \sum_{(w:l)\in\wlita(A), l\neq p, I \models l} s\cdot w < s\cdot (b+\wt(p,A))$,
where $s:=\frac{\wt(p,A)}{|\wt(p,A)|}$
and elements $s\cdot w : l$ or $-s\cdot w : l$ such that
$s\cdot w$ or $-s\cdot w$ is negative can be mapped to
$-s\cdot w : \naf l$ or $s\cdot w : \naf l$
when adding $|w|$ to the respective bound $s\cdot b$ or $-s\cdot (b+\wt(p,A))$.}
\end{comment}
In turn, when $A$ is non-recursive or stratified \cite{DBLP:journals/ai/FaberPL11},
the corresponding subprogram $\mon(A,\emptyset)$
does not introduce disjunction or % as well as
any fresh propositional atom different from $\aux$.

\subsection{Properties}\label{sec:properties}

Our first result generalizes a property of models
of reducts to programs with aggregates.
\begin{restatable}{proposition}{PropComponent}\label{PropComponent}
Let $\Pi$ be a program, $I$ be a model of~$\Pi$, and $J\subset I$ be a model of $F(I,\Pi)$.
Then, there is some component~$C$ of~$\Pi$ such that
$I\cap (C\setminus J)\neq\emptyset$ and
$I\setminus (C\setminus J)\models F(I,\Pi)$.
\end{restatable}
In other words, when any strict subset~$J$ of a model~$I$ of~$\Pi$ satisfies $F(I,\Pi)$,
then there is a model~$K$ of $F(I,\Pi)$ such that $J\subseteq K\subset I$ and
$I\setminus K\subseteq C$ for some component~$C$ of~$\Pi$.
For instance, the model $\{x_1\}$ of $F(\{x_1, \mathit{unequal}, y_1, y_2\},\Pi_2)$,
given in Example~\ref{ex:wrong}, is such that 
$\{x_1, \mathit{unequal}, y_1, y_2\}\setminus\{x_1\}\subseteq C$
for the component $C:=\{\mathit{unequal}, y_1, y_2,\linebreak[1]\SUM[1 : x_1,\linebreak[1] 2 : x_2, 2 : y_1, 3 : y_2] > 5\}$
of $\Pi_2$.

For a program~$\Pi$ and~$A$ of the form~(\ref{eq:aggregate-normal}) such that $\odot \in \{>,\neq\}$, % the
rewritings $\rew(\Pi,A,\V)$ and $\rew(\Pi,A,\rec(\Pi,A))$
have been investigated above.
In order to establish their correctness,
we show that
$\Pi\equiv_{\At(\Pi)}\rew(\Pi,A,V)$
holds for all subsets~$V$ of~$\V$ such that $\rec(\Pi,A)\subseteq V$.
To this end, let
$\AtF(A,V):=\{p^F \mid (p^F \leftarrow \aux) \in \mon(A,V)\}$
denote the fresh, hidden atoms~$p^F$ introduced in $\mon(A,V)$.
Given an interpretation~$I$ (such that $I\cap(\{\aux\} \cup \AtF(A,V)) = \emptyset$)
and $J\subseteq\nolinebreak I$,
we define an \emph{extension} of~$J$ relative to~$I$~by:% % as follows:
\begin{equation*}
 \ext(J,I) := \left\{
    \begin{array}{ll}
        J \cup \{p^F \in \AtF(A,V) \mid p \notin I\} & \mbox{if } I \not\models A \\
        J \cup \{p^F \in \AtF(A,V) \mid p \notin J\} & \mbox{if } I \models A \text{ and } J \not\models F(I,A) \\
        J \cup \{\aux\} \cup \AtF(A,V)               & \mbox{if } I \models A \text{ and } J \models F(I,A) % .
    \end{array}
 \right.
\end{equation*}
For instance, considering $A:=\SUM[1 : x_1, 2 : x_2, 2 : y_1, 3 : y_2] \neq 5$,
$V:=\{\mathit{unequal},\linebreak[1]y_1,y_2\}$,
% $\AtF(A,V):=\{y_1^F,y_2^F\}$,
$I:=\{x_2, \mathit{unequal}, y_1, y_2\}$, and
$J:=\{x_2, y_2\}$,
in view of $I\models A$ and $J\not\models F(I,A)$,
we obtain
$\ext(I,I)=I\cup\{\aux, y_1^F, y_2^F\}$ and  % \{x_2, \mathit{unequal}, y_1, y_2, \aux, y_1^F, y_2^F\}$ and
$\ext(J,I)=J\cup\{y_1^F\}$. % \{x_2, y_2, y_1^F\}$.

For $I$ and $J$ as above,
the following technical lemma yields
$\ext(I,I)$ as the subset-minimal model of
reducts $F(I',\mon(A,V))$ with respect to models~$I'$
of the subprogram $\mon(A,V)$ that extend~$I$.
Under the assumption that a nonempty difference $I\setminus J$
remains local to a component~$C$ of~$\Pi$ such that
some atom in~$C$ depends on~$A$, $\ext(J,I)$ further
constitutes the subset-minimal extension of~$J$ to a model
of $F(\ext(I,I),\mon(A,V))$.
\begin{restatable}{lemma}{LemModular}\label{LemReduct}
Let $\Pi$ be a program,
$A$ be an aggregate, and
$V$ be a set of propositional atoms such that $\rec(\Pi,A)\subseteq V$.
Let $I$ be an interpretation such that $I\cap(\{\aux\} \cup \AtF(A,V)) = \emptyset$
and  $J\subseteq I$.
Then, the following conditions hold:
\begin{enumerate}[leftmargin=*]
\item For any model~$I'$ of $\mon(A,V)$ such that $I'\setminus(\{\aux\} \cup \AtF(A,V)) = I$, we have that
      $\ext(I,I)\subseteq I'$ and $\ext(I,I)\models F(I',\mon(A,V))$.
\item If $J=I$ or $I\setminus J\subseteq C$ for some component~$C$ of~$\Pi$
      such that there is a rule $r\in\Pi$ with $H(r)\cap C\neq\emptyset$ and $A\in B(r)$,
      then $\ext(J,I)\models F(\ext(I,I),\mon(A,V))$ and
      $\ext(J,I)\subseteq J'$ for any model~$J'$ of $F(\ext(I,I),\mon(A,V))$
      such that $J'\setminus(\{\aux\} \cup \AtF(A,V)) = J$.%
\end{enumerate}
\end{restatable}

With the auxiliary result describing the formation of models of
$\mon(A,V)$ and its reducts at hand, we can show the main result of this paper
that the presented rewritings preserve the stable models of a program~$\Pi$.
\begin{restatable}{theorem}{ThmCorrectness}\label{thm:correctness}
Let $\Pi$ be a program,
$A$ be an aggregate, and
$V$ be a set of propositional atoms such that $\rec(\Pi,A)\subseteq V$.
Then, it holds that % $\rew(\Pi,A,V)$ is \lparse-like, and 
$\Pi \equiv_{\At(\Pi)} \rew(\Pi,A,V)$.
\end{restatable}

The second objective is establishing the properties of a 
\emph{polynomial, faithful, and modular} translation % function
\cite{DBLP:journals/jancl/Janhunen06}, i.e., a mapping that is polynomial-time computable,
preserves stable models (when auxiliary atoms are ignored),
and can be computed independently on parts of an input program.
The faithfulness of $\rew(\Pi,A,V)$ for any $\rec(\Pi,A)\subseteq V\subseteq\V$
is stated in Theorem~\ref{thm:correctness}.
Moreover, since at most $3\cdot n$ additional rules (\ref{eq:pF:1})--(\ref{eq:pF:3})
are introduced in $\mon(A,V)$ for~$A$ of the form~(\ref{eq:aggregate-normal}),
it is clear that $\rew(\Pi,A,V)$ is polynomial-time computable.
This also holds when applying the strong equivalences (A)--(O) to
replace aggregates by conjunctions,
where the worst cases (N) and~(O) yield a quadratic blow-up.

Hence, the final condition to be addressed is modularity.
Given that the refined rewriting $\rew(\Pi,A,\rec(\Pi,A))$
refers to the components of an entire program~$\Pi$,
this rewriting cannot be done in parts.
The unoptimized rewriting $\rew(\Pi,A,\V)$, however,
consists of the subprogram $\mon(A,\V)$, which is independent of~$\Pi$,
and otherwise merely replaces~$A$ by~$\aux$ in~$\Pi$.
Thus, under the assumption that $A$ does not occur outside of~$\Pi$
(where it cannot be replaced by $\aux$),
$\rew(\Pi,A,\V)$ complies with the modularity condition.
\begin{restatable}{proposition}{ThmModular}\label{ThmModular}
Let $\Pi,\Pi'$ be programs and
$A$ be an aggregate such that $A\notin\Ag(\Pi')$.
Then, it holds that $\rew(\Pi \cup \Pi',A,\V) = \rew(\Pi,A,\V) \cup \Pi'$.
\end{restatable}
Note that $A\notin\Ag(\Pi')$ is not a restriction,
given that an element $w:\bot$ with an arbitrary weight~$w$ can be added
for obtaining a new aggregate~$A'$ that is strongly equivalent to~$A$.
In practice, however, one would rather aim at reusing a propositional atom $\aux$
that represents the satisfaction of~$A$ instead of redoing the rewriting
with another fresh atom $\aux'$.

%%% Local Variables: 
%%% mode: latex
%%% TeX-master: "main"
%%% End: 

\fi

%\section{Experiment}
%\TODO{Maybe we can run an experiment to compare ASP solvers and QBF solvers on instances of GSS.}

\section{Related work}\label{sec:related}

Several semantics were proposed in the literature for interpreting ASP programs with aggregates.
Among them, F-stable model semantics \cite{DBLP:journals/ai/FaberPL11,DBLP:journals/tocl/Ferraris11} was considered in this paper because it is implemented by widely-used ASP solvers \cite{DBLP:journals/tplp/FaberPLDI08,DBLP:journals/ai/GebserKS12}.
Actually, the definition provided in Section~\ref{sec:background} is slightly different than those in \cite{DBLP:journals/ai/FaberPL11,DBLP:journals/tocl/Ferraris11}.
In particular, the language considered in \cite{DBLP:journals/tocl/Ferraris11} has a broader syntax allowing for arbitrary nesting of propositional formulas.
The language considered in \cite{DBLP:journals/ai/FaberPL11}, instead, does not explicitly allow the use of double negation, which however can be simulated by means of auxiliary atoms.
For example, in \cite{DBLP:journals/ai/FaberPL11} a rule $p \leftarrow \naf\naf p$ must be modeled by using a fresh atom $p^F$ and the following subprogram:
$\{p \leftarrow \naf p^F,\ p^F \leftarrow \naf p\}$.
Moreover, in \cite{DBLP:journals/ai/FaberPL11} aggregates cannot contain negated literals, which can be simulated by          auxiliary atoms as well.
On the other hand, negated aggregates are permitted in \cite{DBLP:journals/ai/FaberPL11}, while they are not considered in this paper.
Actually, programs with negated aggregates are those for which \cite{DBLP:journals/tocl/Ferraris11} and \cite{DBLP:journals/ai/FaberPL11} disagree.
As a final remark, the reduct of \cite{DBLP:journals/ai/FaberPL11} does not remove negated literals from satisfied bodies, which however are necessarily true in all counter-models because double negation is not allowed.

Techniques to rewrite logic programs with aggregates into equivalent programs with simpler aggregates were investigated in the literature right from the beginning \cite{DBLP:journals/ai/SimonsNS02}.
In particular, rewritings into \lparse-like programs, which differ from those presented in this paper, were considered in \cite{DBLP:journals/tplp/LiuY13}.
As a general comment, since disjunction is not considered in \cite{DBLP:journals/tplp/LiuY13}, all aggregates causing a jump from the first to the second level of the polynomial hierarchy are excluded a priori.
This is the case for aggregates of the form $\SUM(S) \neq b$, $\AVG(S) \neq b$, and $\COUNT(S) \neq b$, as       noted in \cite{DBLP:journals/tplp/SonP07}, but also for comparators other than $\neq$ when negative weights are involved.
In fact, in \cite{DBLP:journals/tplp/LiuY13} negative weights are eliminated by a rewriting similar to the one in~(\ref{eq:aggregate:rewrite}), but negated literals are introduced instead of auxiliary atoms, which may lead to counterintuitive results \cite{DBLP:journals/tplp/FerrarisL05}.
A different rewriting was presented in \cite{DBLP:journals/tocl/Ferraris11}, whose output are programs with nested expressions, a construct that is not supported by current ASP systems.
Other relevant rewriting techniques were proposed in \cite{DBLP:conf/lpnmr/BomansonJ13,DBLP:conf/jelia/BomansonGJ14}, and proved to be quite efficient in practice.
However, these rewritings produce aggregate-free programs preserving F-stable models only in the stratified case, or if recursion is limited to convex aggregates.
On the other hand, it is interesting to observe that the rewritings of \cite{DBLP:conf/lpnmr/BomansonJ13,DBLP:conf/jelia/BomansonGJ14} are applicable to the output of the rewritings presented in this paper in order to completely eliminate aggregates, thus  preserving F-stable models in general.

The rewritings given in Section~\ref{sec:compilation} do not apply to other % relevant stable model 
semantics % for logic programs with aggregates
whose stability checks are not based on minimality \cite{DBLP:journals/tplp/PelovDB07,DBLP:journals/tplp/SonP07,DBLP:journals/ai/ShenWEFRKD14}, or whose program reducts do not contain aggregates \cite{DBLP:journals/tplp/GelfondZ14}.
They also disregard
DL \cite{DBLP:journals/ai/EiterILST08} and HEX \cite{DBLP:journals/jair/EiterFKRS14} atoms, extensions of ASP for interacting with external knowledge bases, possibly expressed in different languages, that act semantically similar to aggregate functions.
%
% Aggregate functions are also semantically similar to DL \cite{DBLP:journals/ai/EiterILST08} and HEX \cite{DBLP:journals/jair/EiterFKRS14} atoms, extensions of ASP for interacting with external knowledge bases, possibly expressed in different languages.
% The rewritings of Section~\ref{sec:compilation} do not apply to these languages as well.

As a final remark, the notion of a (positive) dependency graph given in Section~\ref{sec:refined} % provides a significant improvement in the detection of recursion in aggregates.
refines the concept of recursion through aggregates.
In fact, many works \cite{DBLP:journals/jair/AlvianoCFLP11,DBLP:journals/ai/FaberPL11,DBLP:journals/ai/SimonsNS02,DBLP:journals/tplp/SonP07}
consider an aggregate as recursive as soon as aggregated literals
depend on the evaluation of the aggregate.
% in the literature simply consider as recursive aggregates those in which aggregated data depend on the evaluation of the aggregate itself \cite{DBLP:journals/jair/AlvianoCFLP11,DBLP:journals/ai/FaberPL11,DBLP:journals/ai/SimonsNS02,DBLP:journals/tplp/SonP07}.
According to this simple definition, the aggregate in the following program 
% will be considered as recursive:
is deemed to be recursive:
$\{p \leftarrow \SUM[-1 : q] > -1,\ q\leftarrow p\}$.
However,
(negative) recursion through a rule like $p\leftarrow \naf q$
is uncritical for the computation of F-stable models,
% , as the dependency of the aggregate on $p$ is not positive:
as it cannot lead to circular support.
% as the aggregate cannot provide circular support for~$p$.
% and $p$ cannot introduce self support, as $q$ cannot be self supported in $\{q\leftarrow \naf q\}$.
% This property is captured by the notion of 
In fact, the % positive 
         dependency graph introduced in Section~\ref{sec:refined} % , which does not introduce arcs for such non-positive dependencies.
does not include arcs for such non-positive dependencies,
so that strongly connected components render potential
circular support more precisely.
% Another advantage of this approach is that strongly connected components are correctly recognized.
For example, % program $\Pi_3'$ in Fig.~\ref{fig:dep-graph}
the aforementioned program has three components, namely
$\{p\}$, $\{q\}$, and $\{\SUM[-1 : q] > -1\}$.
If the dependency of $\SUM[-1 : q] > -1$ on $q$ were mistakenly
considered as positive, the three components would be joined into one,
thus unnecessarily extending the scope of stability checks.
% and this would affect negatively the performance of any ASP solver.

\section{Conclusion}\label{sec:conclusion}

The representation of knowledge in ASP is eased by the availability of several constructs, among them aggregation functions.
As it is common in combinatorial problems, the structure of input instances is simplified in order to improve the efficiency of low-level reasoning. % the pruning of the search space.
Concerning aggregation functions, the simplified form processed by current ASP solvers is known as monotone, and by complexity arguments faithfulness of current rewritings is subject to specific conditions, i.e., input programs can only contain convex aggregates.
The (unoptimized) translation % function 
presented in this paper is instead polynomial, faithful, and modular for all common aggregation functions, including non-convex instances of \SUM, \AVG, and \COUNT.
Moreover, the rewriting approach extends to aggregation functions such as \MIN, \MAX, \EVEN, and \ODD.
The proposed rewritings are implemented in a prototype system 
% that allows for experimenting with arbitrary recursive aggregates in practice.
and also    adopted   in the recent version~4.5  of the grounder \textsc{gringo}.

\section*{Acknowledgement}
Mario Alviano was partially supported by MIUR within project ``SI-LAB BA2KNOW  -- Business Analitycs to Know'', by Regione Calabria, POR Calabria FESR 2007-2013  within project ``ITravel PLUS'' and project ``KnowRex'', by the National Group for Scientific Computation (GNCS-INDAM), and by Finanziamento Giovani Ricercatori UNICAL.
Martin Gebser was supported by AoF (grant \#251170)
within the Finnish Centre of Excellence in Computational Inference Research (COIN).

\clearpage
% \bibliographystyle{acmtrans}
% \bibliography{bibtex}

%%% Local Variables: 
%%% mode: latex
%%% TeX-master: "main"
%%% End: 

\label{lastpage}

\clearpage
\appendix

\section{Proofs}

\ifInput
\sloppy

\begin{restatable}{proposition}{ThmPreOne}\label{ThmPreOne}
The strong equivalences stated in (A)--(O) hold.
\end{restatable}
\begin{proof}
Let $I$ be an interpretation.
Given that the reducts with respect to~$I$
of expressions on both sides of $\equiv_\V$ in (A)--(O)
fix the same (negative) literals,
it is sufficient to show equivalence.
\begin{enumerate}[leftmargin=*,labelsep=2pt,label=(\Alph{enumi})]
\item
$\SUM[w_1 : l_1, \ldots, w_n : l_n] < b$ is true in $I$ if and only if

$\sum_{i \in [1..n], I \models l_i} w_i < b$ if and only if

$\sum_{i \in [1..n], I \models l_i}-w_i >-b$ if and only if

$\SUM[-w_1 : l_1, \ldots, -w_n : l_n] > -b$ is true in $I$.
\item
$\SUM[w_1 : l_1, \ldots, w_n : l_n] \leq b$ is true in $I$ if and only if

$\sum_{i \in [1..n], I \models l_i} w_i \leq b$ if and only if

$\sum_{i \in [1..n], I \models l_i}-w_i >-b-1$ if and only if

$\SUM[-w_1 : l_1, \ldots, -w_n : l_n] > -b-1$ is true in $I$.
\item
$\SUM[w_1 : l_1, \ldots, w_n : l_n] \geq b$ is true in $I$ if and only if

$\sum_{i \in [1..n], I \models l_i} w_i \geq b$ if and only if

$\sum_{i \in [1..n], I \models l_i} w_i > b-1$ if and only if

$\SUM[w_1 : l_1, \ldots, w_n : l_n] > b-1$ is true in $I$.
\item
$\SUM[w_1 : l_1, \ldots, w_n : l_n] = b$ is true in $I$ if and only if

$\sum_{i \in [1..n], I \models l_i} w_i = b$ if and only if

$\sum_{i \in [1..n], I \models l_i} w_i > b-1$ and 
$\sum_{i \in [1..n], I \models l_i}-w_i > -b-1$ if and only if

$\SUM[w_1 : l_1, \ldots, w_n : l_n] > b-1 \wedge
 \SUM[-w_1 : l_1, \ldots, -w_n : l_n] > -b-1$ is true in $I$.
\item
$\AVG[w_1 : l_1, \ldots, w_n : l_n] \odot b$ is true in $I$ if and only if

$m := |\{i \in [1..n] \mid I \models l_i\}|$, $m \geq 1$, and $\sum_{i \in [1..n],  I \models l_i} \frac{w_i}{m} % w_i / m 
       \odot b$ if and only if

$m := |\{i \in [1..n] \mid I \models l_i\}|$, $m \geq 1$, and $\sum_{i \in [1..n],  I \models l_i} w_i \odot m\cdot b$ if and only if

$\{i \in [1..n] \mid I \models l_i\}\neq\emptyset$ and $\sum_{i \in [1..n],  I \models l_i} (w_i-b) \odot 0$ if and only if

% $\sum_{i \in [1..n],  I \models l_i} (w_i-b) \odot 0$ and $\sum_{i \in [1..n],  I \models l_i} 1 > 0$ if and only if

$\SUM[w_1-b : l_1, \ldots, w_n-b : l_n] \odot 0 \wedge
 \SUM[1:l_1, \ldots, 1:l_n] > 0$ is true in $I$.
\item
$\MIN[w_1 : l_1, \ldots, w_n : l_n] < b$ is true in $I$ if and only if

$\min(\{w_i \mid i \in [1..n],  I \models l_i\} \cup \{+\infty\}) < b$ if and only if

$\{i \in [1..n] \mid w_i < b, I \models l_i\}\neq\emptyset$ if and only if

$\SUM[1:l_i \mid i \in [1..n], w_i < b] > 0$ is true in $I$.
\item
$\MIN[w_1 : l_1, \ldots, w_n : l_n] \leq b$ is true in $I$ if and only if

$\min(\{w_i \mid i \in [1..n],  I \models l_i\} \cup \{+\infty\}) \leq b$ if and only if

$\{i \in [1..n] \mid w_i \leq b, I \models l_i\}\neq\emptyset$ if and only if

$\SUM[1:l_i \mid i \in [1..n], w_i \leq b] > 0$ is true in $I$.
\item
$\MIN[w_1 : l_1, \ldots, w_n : l_n] \geq b$ is true in $I$ if and only if

$\min(\{w_i \mid i \in [1..n],  I \models l_i\} \cup \{+\infty\}) \geq b$ if and only if

$\{i \in [1..n] \mid w_i < b, I \models l_i\}=\emptyset$ if and only if

$\SUM[-1:l_i \mid i \in [1..n], w_i < b] > -1$ is true in $I$.
\item
$\MIN[w_1 : l_1, \ldots, w_n : l_n] > b$ is true in $I$ if and only if

$\min(\{w_i \mid i \in [1..n],  I \models l_i\} \cup \{+\infty\}) > b$ if and only if

$\{i \in [1..n] \mid w_i \leq b, I \models l_i\}=\emptyset$ if and only if

$\SUM[-1:l_i \mid i \in [1..n], w_i \leq b] > -1$ is true in $I$.
\item
$\MIN[w_1 : l_1, \ldots, w_n : l_n] = b$ is true in $I$ if and only if

$\min(\{w_i \mid i \in [1..n],  I \models l_i\} \cup \{+\infty\}) = b$ if and only if

$\{i \in [1..n] \mid w_i = b, I \models l_i\}\neq\emptyset$ and
$\{i \in [1..n] \mid w_i < b, I \models l_i\}=\emptyset$ if and only if

$\SUM[1-n\cdot(b-w_i) : l_i \mid i \in [1..n], w_i \leq b] > 0$ is true in $I$.
\item
$\MIN[w_1 : l_1, \ldots, w_n : l_n] \neq b$ is true in $I$ if and only if

$\min(\{w_i \mid i \in [1..n],  I \models l_i\} \cup \{+\infty\}) \neq b$ if and only if

$\{i \in [1..n] \mid w_i = b, I \models l_i\}=\emptyset$ or
$\{i \in [1..n] \mid w_i < b, I \models l_i\}\neq\emptyset$ if and only if

$\SUM[n\cdot(b-w_i)-1 : l_i \mid i \in [1..n], w_i \leq b] > -1$ is true in $I$.
\item
$\MAX[w_1 : l_1, \ldots, w_n : l_n] \odot b$ is true in $I$ if and only if

$\max(\{w_i \mid i \in [1..n], I \models l_i\} \cup \{-\infty\}) \odot b$ if and only if

$\min(\{-w_i \mid i \in [1..n],  I \models l_i\} \cup \{+\infty\}) \ f(\odot)\ {-b}$ if and only if

$\MIN[-w_1 : l_1, \ldots, -w_n : l_n] \ f(\odot)\ {-b}$ is true in $I$,

where $<{} \stackrel{f}{\mapsto} {}>$, $\leq{} \stackrel{f}{\mapsto} {}\geq$, $\geq{} \stackrel{f}{\mapsto} {}\leq$, $>{} \stackrel{f}{\mapsto} {}<$, $={} \stackrel{f}{\mapsto} {}=$, and $\neq{} \stackrel{f}{\mapsto} {}\neq$.
\item
$\COUNT[l_1, \ldots, l_n] \odot b$ is true in $I$ if and only if

$|\{i \in [1..n] \mid I \models l_i\}| \odot b$ if and only if

$\SUM[1 : l_1, \ldots, 1 : l_n] \odot b$ is true in $I$.
\item
$\EVEN[l_1, \ldots, l_n]$ is true in $I$ if and only if

$|\{i \in [1..n] \mid I \models l_i\}|$ is an even number if and only if

$|\{i \in [1..n] \mid I \models l_i\}|\neq 2 \cdot i' - 1$ for all $i'\in [1..\left\lceil n/2 \right\rceil]$ if and only if

$\bigwedge_{i \in [1..\left\lceil n/2 \right\rceil]} \SUM[1:l_1, \ldots, 1:l_n] \neq 2 \cdot i - 1$ is true in $I$.
\item
$\ODD[l_1, \ldots, l_n]$ is true in $I$ if and only if

$|\{i \in [1..n] \mid I \models l_i\}|$ is an odd number if and only if

$|\{i \in [1..n] \mid I \models l_i\}|\neq 2 \cdot i'$ for all $i' \in [0..\left\lfloor n/2 \right\rfloor]$ if and only if

$\bigwedge_{i \in [0..\left\lfloor n/2 \right\rfloor]} \SUM[1:l_1, \ldots, 1:l_n] \neq 2 \cdot i$ is true in $I$.
\end{enumerate}
\end{proof}

\PropComponent*
\begin{proof}
For any components $C_1\neq C_2$ of~$\Pi$,
the existence of some path from $\alpha_1\in C_1$ to $\beta_2\in C_2$ in $\G_\Pi$
implies that there is no path from any $\alpha_2\in C_2$ to $\beta_1\in C_1$ in $\G_\Pi$.
Hence,
since $J\subset I$ and $\G_\Pi$ is finite, there is some component~$C$ of~$\Pi$
such that $I\cap (C\setminus J)\neq\emptyset$ and $\beta\in C\setminus J$ for any path from
$\alpha\in C$ to $\beta\in I\setminus J$ in $\G_\Pi$.
Consider any rule~$r\in F(I,\Pi)$ such that $H(r)\cap (I\setminus (C\setminus J)) = \emptyset$.
Then, $I\models B(r)$ and $I\models r$ yield that $H(r)\cap I\neq\emptyset$,
so that $H(r)\cap C\neq\emptyset$.
On the other hand, since $J\subseteq I\setminus (C\setminus J)$,
we have that $H(r)\cap J=\emptyset$,
which together with $J\models r$ implies that $J\not\models B(r)$.
That is, there is some positive literal $\beta\in B(r)$ such that
$I\models \beta$, $J\not\models F(I,\beta)$,
some $\alpha \in C$ has an arc to~$\beta$ in $\G_\Pi$,
and one of the following three cases applies:
\begin{enumerate}[leftmargin=*]
\item
If $\beta\in I\setminus J$, then $\beta\in C\setminus J$,
so that $I\setminus (C\setminus J)\not\models B(r)$ and $I\setminus (C\setminus J)\models r$.
\item
If $\beta$ is an aggregate of the form~(\ref{eq:aggregate-normal}) such that $\odot$ is $>$,
for any $p\in I\setminus J$ such that $(w:p)\in\wlitp(\beta)$,
we have that $p\in C\setminus J$ because some $\alpha\in C$ has a path to~$p$ in~$\G_\Pi$.
Along with $J\subset I$,
this in turn yields that
\begin{equation*}
\mbox{$\sum$}_{(w:p)\in\wlitp(\beta),p\in I\setminus (C\setminus J)}w = \mbox{$\sum$}_{(w:p)\in\wlitp(\beta),p\in J}w\text{.}
\end{equation*}
Moreover, $J\subseteq I\setminus (C\setminus J)$ implies that
\begin{equation*}
\mbox{$\sum$}_{(w:p)\in\wlitn(\beta),p\in I\setminus (C\setminus J)}w \leq \mbox{$\sum$}_{(w:p)\in\wlitn(\beta),p\in J}w\text{,}
\end{equation*}
so that
\begin{equation*}
\begin{array}{@{}r@{}l@{}l}
& \mbox{$\sum$}_{(w:p)\in\wlita(\beta),p\in I\setminus (C\setminus J)}w
\\
{} = {} &
\mbox{$\sum$}_{(w:p)\in\wlitp(\beta),p\in I\setminus (C\setminus J)}w & {} + \mbox{$\sum$}_{(w:p)\in\wlitn(\beta),p\in I\setminus (C\setminus J)}w
\\
{} \leq {} &
\mbox{$\sum$}_{(w:p)\in\wlitp(\beta),p\in J}w & {} + \mbox{$\sum$}_{(w:p)\in\wlitn(\beta),p\in J}w
\\ 
{} = {} &
\mbox{$\sum$}_{(w:p)\in\wlita(\beta),p\in J}w\text{.}
\end{array}
\end{equation*}
In view of $J\not\models F(I,\beta)$, we further conclude that
\begin{equation*}
\begin{array}{@{}r@{}l@{}l}
& \multicolumn{2}{@{}l}{\mbox{$\sum$}_{(w:l)\in\wlita(\beta),I\setminus (C\setminus J)\models F(I,l)}w}
\\ 
{} = {} &
\mbox{$\sum$}_{(w:p)\in\wlita(\beta),p\in I\setminus (C\setminus J)}w & {} + \mbox{$\sum$}_{(w:l)\in\wlita(\beta),l\notin\V,I\models l}w
\\ 
{} \leq {} &
\mbox{$\sum$}_{(w:p)\in\wlita(\beta),p\in J}w & {} + \mbox{$\sum$}_{(w:l)\in\wlita(\beta),l\notin\V,I\models l}w
\\
{} = {} &
\mbox{$\sum$}_{(w:l)\in\wlita(\beta),J\models F(I,l)}w
\\
{} \leq {} &
b\text{.}
\end{array}
\end{equation*}
That is, $I\setminus (C\setminus J)\not\models F(I,\beta)$,
so that $I\setminus (C\setminus J)\not\models B(r)$ and $I\setminus (C\setminus J)\models r$.
\item
If $\beta$ is an aggregate of the form~(\ref{eq:aggregate-normal}) such that $\odot$ is $\neq$,
for any $p\in I\setminus J$ such that $(w:p)\in\wlita(\beta)$,
we have that $p\in C\setminus J$ because some $\alpha\in C$ has a path to~$p$ in~$\G_\Pi$.
Along with $J\subset I$,
this in turn yields that
\begin{equation*}
\mbox{$\sum$}_{(w:p)\in\wlita(\beta),p\in I\setminus (C\setminus J)}w = \mbox{$\sum$}_{(w:p)\in\wlita(\beta),p\in J}w\text{.}
\end{equation*}
In view of $J\not\models F(I,\beta)$, we further conclude that
\begin{equation*}
\begin{array}{@{}r@{}l@{}l}
& \multicolumn{2}{@{}l}{\mbox{$\sum$}_{(w:l)\in\wlita(\beta),I\setminus (C\setminus J)\models F(I,l)}w}
\\ 
{} = {} &
\mbox{$\sum$}_{(w:p)\in\wlita(\beta),p\in I\setminus (C\setminus J)}w & {} + \mbox{$\sum$}_{(w:l)\in\wlita(\beta),l\notin\V,I\models l}w
\\ 
{} = {} &
\mbox{$\sum$}_{(w:p)\in\wlita(\beta),p\in J}w & {} + \mbox{$\sum$}_{(w:l)\in\wlita(\beta),l\notin\V,I\models l}w
\\
{} = {} &
\mbox{$\sum$}_{(w:l)\in\wlita(\beta),J\models F(I,l)}w
\\
{} = {} &
b\text{.}
\end{array}
\end{equation*}
That is, $I\setminus (C\setminus J)\not\models F(I,\beta)$,
so that $I\setminus (C\setminus J)\not\models B(r)$ and $I\setminus (C\setminus J)\models r$.
\end{enumerate}
Since $I\setminus (C\setminus J)\models r$ also holds for
any rule~$r\in F(I,\Pi)$ such that $H(r)\cap (I\setminus (C\setminus J)) \neq \emptyset$,
we have shown that $I\setminus (C\setminus J)\models F(I,\Pi)$.
\end{proof}

\LemModular*
\begin{proof}
\begin{enumerate}[leftmargin=*]
\item 
Let $I'\models\mon(A,V)$ such that $I'\setminus(\{\aux\} \cup \AtF(A,V)) = I$.
Then, in view of
$\{p^F\leftarrow\nolinebreak\naf p \mid p^F\in\AtF(A,V)\}\subseteq\mon(A,V)$,
we have that $\{p^F\in\AtF(A,V) \mid p\notin I\}\subseteq I'$,
so that
$I'\models\{p\vee p^F\leftarrow\naf\naf\aux \mid p^F\in\AtF(A,V)\}$.
Moreover, when $\aux\in I'$,
$\{p^F\leftarrow\nolinebreak\aux \mid p^F\in\AtF(A,V)\}\subseteq\mon(A,V)$
yields that
$\AtF(A,V)\subseteq I'$.

($>$) \
For $A$ of the form~(\ref{eq:aggregate-normal}) such that $\odot$ is $>$,
let $A'$ be the body of rule~(\ref{eq:aggregate:rewrite}) from $\mon(A,V)$.
% the body of rule~(\ref{eq:aggregate:rewrite}) from $\mon(A,V)$ is true in~$I'$
Then, we have that $I'\models A'$ if and only if
\begin{equation}\label{eq:p1}
\begin{array}{@{}r@{}l}
& \mbox{$\sum$}_{(w:l)\in\wlitp(A), I\models l}w
-
\mbox{$\sum$}_{(w:l)\in\wlitn(A), l \notin V, I\not\models l}w
\\
{} - {} &
\mbox{$\sum$}_{(w:p)\in\wlitn(A), p \in V, p^F\in I'}w
>
b
-
\mbox{$\sum$}_{(w:l)\in\wlitn(A)}w
\text{.}
\end{array}
\end{equation}
By adding $\mbox{$\sum$}_{(w:l)\in\wlitn(A)}w$ on both sides,
(\ref{eq:p1}) yields
\begin{equation}\label{eq:p2}
\begin{array}{@{}r@{}l}
& \mbox{$\sum$}_{(w:l)\in\wlitp(A), I\models l}w
+
\mbox{$\sum$}_{(w:l)\in\wlitn(A), l \notin V, I\models l}w
\\
{} + {} &
\mbox{$\sum$}_{(w:p)\in\wlitn(A), p \in V, p^F\notin I'}w
>
b
\text{.}
\end{array}
\end{equation}
Since 
$\{p \in V\mid (w:p)\in\wlitn(A), p^F\notin I'\} \subseteq I$,
% $\AtF(A,V)\setminus I'\subseteq\{p^F\in\AtF(A,V) \mid p\in I\}$,
$I\models A$ implies that~(\ref{eq:p2}) holds
and $\{\aux\}\cup\AtF(A,V)\subseteq I'$,
but (\ref{eq:p2}) does not hold for
$I'=I\cup\{p^F\in\nolinebreak\AtF(A,V) \mid p\notin I\}$
otherwise.
In either case,
we have that
$\ext(I,I)\subseteq I'$ and
$\ext(I,I)\models F(I',\mon(A,V))$,
where
$\ext(I,I)\models A'$ if and only if $I\models A$.

($\neq$) \
For $A$ of the form~(\ref{eq:aggregate-normal}) such that $\odot$ is $\neq$,
(\ref{eq:p2}) holds when
$I\models\gtA$, where $\gtA := \SUM[w_1 :\nolinebreak l_1,\linebreak[1] \ldots, w_n : l_n] > b$.
Moreover, for $\ltA := \SUM[-w_1 : l_1, \ldots,\linebreak[1] -w_n :\nolinebreak l_n] > -b$,
let $\ltA'$ be the body of rule~(\ref{eq:aggregate:rewrite}) from $\mon(\ltA,V)$.
Then,
we have that $I'\models\ltA'$ if and only if
\begin{equation}\label{eq:p3}
\begin{array}{@{}r@{}l}
& \mbox{$\sum$}_{(w:l)\in\wlitp(A), l \notin V, I\not\models l}w
-
\mbox{$\sum$}_{(w:l)\in\wlitn(A), I\models l}w
\\
{} + {} &
\mbox{$\sum$}_{(w:p)\in\wlitp(A), p \in V, p^F\in I'}w
>
\mbox{$\sum$}_{(w:l)\in\wlitp(A)}w
-b
\text{.}
\end{array}
\end{equation}
By subtracting $\mbox{$\sum$}_{(w:l)\in\wlitp(A)}w$ on both sides and multiplying with~$-1$,
(\ref{eq:p3}) yields
\begin{equation}\label{eq:p4}
\begin{array}{@{}r@{}l}
& \mbox{$\sum$}_{(w:l)\in\wlitp(A), l \notin V, I\models l}w
+
\mbox{$\sum$}_{(w:l)\in\wlitn(A), I\models l}w
\\
{} + {} &
\mbox{$\sum$}_{(w:p)\in\wlitp(A), p \in V, p^F\notin I'}w
<
b
\text{.}
\end{array}
\end{equation}
Since 
$\{p \in V\mid (w:p)\in\wlitp(A), p^F\notin I'\} \subseteq I$,
% $\AtF(A,V)\setminus I'\subseteq\{p^F\in\AtF(A,V) \mid p\in I\}$,
$I\models \ltA$ implies that~(\ref{eq:p4}) holds
and $\{\aux\}\cup\AtF(A,V)\subseteq I'$,
but (\ref{eq:p4}) does not hold for
$I'=I\cup\{p^F\in\nolinebreak\AtF(A,V) \mid p\notin I\}$
otherwise.
In view of $I\models A$ if and only if $I\models\gtA$ or $I\models\ltA$,
we further conclude that
$\ext(I,I)=I\cup\{\aux\}\cup\AtF(A,V)\subseteq I'$ when $I\models A$,
while neither~(\ref{eq:p2}) nor~(\ref{eq:p4}) holds for
$I'=I\cup\{p^F\in\nolinebreak\AtF(A,V) \mid p\notin I\}=\ext(I,I)$ when $I\not\models A$.
In either case,
we have that
$\ext(I,I)\subseteq I'$ and
$\ext(I,I)\models F(I',\mon(A,V))$.
\item
Assume that $J=I$ or $I\setminus J\subseteq C$ for some component~$C$ of~$\Pi$
such that there is a rule $r\in\Pi$ with $H(r)\cap C\neq\emptyset$ and $A\in B(r)$,
and let $J'\models F(\ext(I,I),\mon(A,V))$
such that $J'\setminus(\{\aux\} \cup \AtF(A,V)) = J$.
Then, in view of
$\{p^F\leftarrow\top \mid p^F\in\AtF(A,V),\linebreak[1]p\notin\nolinebreak I\}\subseteq F(\ext(I,I),\mon(A,V))$,
we have that $\{p^F\in\AtF(A,V) \mid p\notin I\}\subseteq\nolinebreak J'$.
When $I\not\models A$,
then $F(\ext(I,I),\mon(A,V))=\{p^F\leftarrow\top \mid p^F\in\AtF(A,V),\linebreak[1]p\notin\nolinebreak I\}$,
$\ext(J,I)=J\cup\{p^F\in\AtF(A,V)\mid p\notin I\}\subseteq J'$, and
$\ext(J,I)\models F(\ext(I,I),\linebreak[1]\mon(A,V))$.
Below assume that $I\models A$, so that
$\{p\vee\nolinebreak p^F\leftarrow\nolinebreak\top \mid\linebreak[1] p^F\in\nolinebreak\AtF(A,V)\}\subseteq F(\ext(I,I),\mon(A,V))$
implies % that
$\{p^F\in\AtF(A,V) \mid p\notin J\}\subseteq J'$.

($>$) \
For $A$ of the form~(\ref{eq:aggregate-normal}) such that $\odot$ is $>$,
let $A'$ be the body of rule~(\ref{eq:aggregate:rewrite}) from $\mon(A,V)$.
Then,
(\ref{eq:p2}) yields that
$J'\models F(\ext(I,I),A')$ if and only if
\begin{equation}\label{eq:p5}
\begin{array}{@{}r@{}l@{}l}
& \mbox{$\sum$}_{(w:l)\in\wlitp(A), l \notin \V, I\models l}w
& {} +
\mbox{$\sum$}_{(w:l)\in\wlitn(A), l \notin V, I\models l}w
\\
{} + {} &
\mbox{$\sum$}_{(w:p)\in\wlitp(A), p \in J}w
& {} + 
\mbox{$\sum$}_{(w:p)\in\wlitn(A), p \in V, p^F\notin J'}w
>
b
\text{.}
\end{array}
\end{equation}
%
% Note that
% %
% \begin{equation}\label{eq:pp1}
% \mbox{$\sum$}_{(w:l)\in\wlitn(A), l \notin V, J\models F(I,l)}w
% \geq
% \mbox{$\sum$}_{(w:l)\in\wlitn(A), l \notin V, I\models l}w
% \text{,}
% \end{equation}
%
Moreover,
% \begin{equation*}
$
 \{p \in V\mid (w:p)\in\wlitn(A), p^F\notin J'\} 
 \subseteq J
 \subseteq I
% \{p^F\in\AtF(A,V) \mid p\in J\}
% \subseteq
% \{p^F\in\nolinebreak\AtF(A,V) \mid\linebreak[1] p\in I\}
$
% \end{equation*}
implies that
\begin{equation}\label{eq:pp2}
% \begin{array}{@{}r@{}l@{}l}
\begin{array}{@{}r@{}l}
% &
% \mbox{$\sum$}_{(w:l)\in\wlitn(A), l \notin V, I\models l}w
% & {} + 
\mbox{$\sum$}_{(w:p)\in\wlitn(A), p \in V, p^F\notin J'}w
% \\
& {} \geq
% {} \geq {} &
% \mbox{$\sum$}_{(w:l)\in\wlitn(A), l \notin V, I\models l}w
% & {} + 
% \geq
\mbox{$\sum$}_{(w:p)\in\wlitn(A), p \in V\cap J}w
\\
& {} \geq
% {} \geq {} &
% \mbox{$\sum$}_{(w:l)\in\wlitn(A), l \notin V, I\models l}w
% & {} + 
\mbox{$\sum$}_{(w:p)\in\wlitn(A), p \in V\cap I}w
% \\
% {} = {} &
% \multicolumn{2}{@{}l}{\mbox{$\sum$}_{(w:l)\in\wlitn(A),I\models l}w}
\text{.}
\end{array}
\end{equation}
In view of the prerequisite that $J=I$ or
$I\setminus J\subseteq C$ for some component~$C$ of~$\Pi$ such that 
there is a rule $r\in\Pi$ with $H(r)\cap C\neq\emptyset$ and $A\in B(r)$,
if $J\subset I$,
some $\alpha \in C$ has an arc to~$A$ in $\G_\Pi$.
Along with $\rec(\Pi,A)\subseteq V$,
this yields that 
$\{p\in\nolinebreak I\setminus\nolinebreak J \mid\linebreak[1] (w:p)\in\wlitp(A)\}=\emptyset$ or
$\{p\in I\setminus J \mid (w:p)\in\wlitn(A)\}\subseteq\nolinebreak V$.
In the former case,
$I\models\nolinebreak A$,
% when (\ref{eq:p2}) holds,
$\mbox{$\sum$}_{(w:p)\in\wlitp(A), p \in J}w
=\mbox{$\sum$}_{(w:p)\in\wlitp(A), p \in I}w$, % , (\ref{eq:p2}), (\ref{eq:pp1}), 
$\mbox{$\sum$}_{(w:l)\in\wlitn(A),l\notin V, J\models F(I,l)}w
\geq
\mbox{$\sum$}_{(w:l)\in\wlitn(A),l\notin V, I\models l}w$, % (\ref{eq:p2}), 
and~(\ref{eq:pp2})
imply that
$J\models F(I,A)$ and
(\ref{eq:p5}) hold.
Moreover,
if $\{p\in I\setminus J \mid (w:p)\in\nolinebreak\wlitn(A)\}\subseteq\nolinebreak V$,
then
$\mbox{$\sum$}_{(w:l)\in\wlitn(A), l \notin V, J\models F(I,l)}w
=\mbox{$\sum$}_{(w:l)\in\wlitn(A), l \notin V, I\models l}w$
% , (\ref{eq:p2}), 
and~(\ref{eq:pp2})
yield that 
(\ref{eq:p5}) holds when $J\models F(I,A)$,
but (\ref{eq:p5}) does not hold for
$J'=J\cup\{p^F\in\nolinebreak\AtF(A,V) \mid p\notin\nolinebreak J\}$ % =\ext(J,I)$
otherwise.
In view of $\ext(I,I)\models A'$ if and only if $I\models A$,
we further conclude that
$\ext(J,I)=J\cup\{\aux\}\cup\AtF(A,V)\subseteq J'$ when $J\models F(I,A)$,
while~(\ref{eq:p5}) does not hold for
$J'=J\cup\{p^F\in\AtF(A,V) \mid p\notin J\}=\ext(J,I)$
when $J\not\models F(I,A)$.
In either case,
we have that
$\ext(J,I)\subseteq J'$ and
$\ext(J,I)\models F(\ext(I,I),\mon(A,V))$.
%
% We have thus shown that 
% $I\models A$ and $J\models F(I,A)$
% imply that $\{\aux\}\cup\AtF(A,V)\subseteq J'$ and 
% $\ext(J,I)\subseteq J'$.
% In addition, $\ext(J,I)\models F(\ext(I,I),\mon(A,V))$
% follows from
% $\{p^F\in\AtF(A,V) \mid p\notin J\}\subseteq \ext(J,I)$ along with
% $\ext(J,I)= J\cup\{\aux\}\cup \AtF(A,V)$
% when (\ref{eq:p5}) holds.

($\neq$) \
For $A$ of the form~(\ref{eq:aggregate-normal}) such that $\odot$ is $\neq$,
(\ref{eq:p5}) holds when $J\models F(I,\gtA)$,
where
$\gtA := \SUM[w_1 :\nolinebreak l_1,\linebreak[1] \ldots, w_n : l_n] > b$.
Moreover, for
$\ltA := \SUM[-w_1 : l_1, \ldots,\linebreak[1] -w_n :\nolinebreak l_n] > -b$,
let $\ltA'$ be the body of rule~(\ref{eq:aggregate:rewrite}) from $\mon(\ltA,V)$.
Then, (\ref{eq:p4}) yields that
$J'\models F(\ext(I,I),\ltA')$ if and only if
\begin{equation}\label{eq:p6}
\begin{array}{@{}r@{}l@{}l}
& \mbox{$\sum$}_{(w:l)\in\wlitp(A), l \notin V, I\models l}w
& {} +
\mbox{$\sum$}_{(w:l)\in\wlitn(A),l\notin\V, I\models l}w
\\
{} + {} &
\mbox{$\sum$}_{(w:p)\in\wlitp(A), p \in V, p^F\notin J'}w
& {} +
\mbox{$\sum$}_{(w:p)\in\wlitn(A), p\in J}w
<
b
\text{.}
\end{array}
\end{equation}
Dual to % (\ref{eq:pp1}) and 
(\ref{eq:pp2}) above,
$
 \{p \in V \mid (w:p)\in\wlitp(A), p^F\notin J'\} 
 \subseteq
 J
$
implies that
\pagebreak[1]
\begin{equation}\label{eq:pp4}
%\begin{array}{@{}r@{}l}
\mbox{$\sum$}_{(w:p)\in\wlitp(A), p \in V, p^F\notin J'}w
\leq
%& {} \leq
\mbox{$\sum$}_{(w:p)\in\wlitp(A), p \in V\cap J}w
%\\
%& {} \leq
%\mbox{$\sum$}_{(w:p)\in\wlitp(A), p \in V\cap I}w
\text{.}
%\end{array}
\end{equation}
In view of the prerequisite regarding $I\setminus J$ and since
$\rec(\Pi,A)\subseteq V$,
we have that $\{p\in I\setminus J \mid (w:p)\in\wlita(A)\}\subseteq V$.
Hence,
$\mbox{$\sum$}_{(w:l)\in\wlitp(A), l \notin V, J\models F(I,l)}w
=\mbox{$\sum$}_{(w:l)\in\wlitp(A), l \notin V, I\models l}w$
% , (\ref{eq:p4}), 
and~(\ref{eq:pp4})
yield that 
(\ref{eq:p6}) holds when $J\models F(I,\ltA)$,
but (\ref{eq:p6}) does not hold for
$J'=J\cup\{p^F\in\AtF(A,V) \mid p\notin J\}$
otherwise.
Moreover,
note that $\ext(I,I)\not\models \ltA'$
implies that
\pagebreak[1]
\begin{equation*}
\begin{array}{@{}r@{}l@{}l}
&
% \mbox{$\sum$}_{(w:l)\in\wlita(A), J\models F(I,l)}w
% \\
% {} = {} &
\mbox{$\sum$}_{(w:l)\in\wlitp(A), J\models F(I,l)}w
& {} +
\mbox{$\sum$}_{(w:l)\in\wlitn(A), J\models F(I,l)}w
\\
{} \geq {} &
\mbox{$\sum$}_{(w:l)\in\wlitp(A), l\notin V, J\models F(I,l)}w
& {} +
\mbox{$\sum$}_{(w:l)\in\wlitn(A), J\models F(I,l)}w
\\
{} = {} &
\mbox{$\sum$}_{(w:l)\in\wlitp(A), l\notin V, I\models l}w
& {} +
\mbox{$\sum$}_{(w:l)\in\wlitn(A), J\models F(I,l)}w
\\
{} \geq {} &
\mbox{$\sum$}_{(w:l)\in\wlitp(A), l\notin V, I\models l}w
& {} +
\mbox{$\sum$}_{(w:l)\in\wlitn(A), I\models l}w
\\
{} \geq {} &
b
\text{,}
\end{array}
\end{equation*}
so that $J\not\models F(I,\ltA)$.
Similarly,
for the body $\gtA'$ of rule~(\ref{eq:aggregate:rewrite}) from $\mon(\gtA,V)$,
since
$\mbox{$\sum$}_{(w:l)\in\wlitn(A), l \notin V, J\models F(I,l)}w
=\mbox{$\sum$}_{(w:l)\in\wlitn(A), l \notin V, I\models l}w$,
% we also have that
$\ext(I,I)\not\models \gtA'$ yields
\begin{equation*}
\begin{array}{@{}r@{}l@{}l}
&
% \mbox{$\sum$}_{(w:l)\in\wlita(A), J\models F(I,l)}w
% \\
% {} = {} &
\mbox{$\sum$}_{(w:l)\in\wlitp(A), J\models F(I,l)}w
& {} +
\mbox{$\sum$}_{(w:l)\in\wlitn(A), J\models F(I,l)}w
\\
{} \leq {} &
\mbox{$\sum$}_{(w:l)\in\wlitp(A), J\models F(I,l)}w
& {} +
\mbox{$\sum$}_{(w:l)\in\wlitn(A), l\notin V, J\models F(I,l)}w
\\
{} = {} &
\mbox{$\sum$}_{(w:l)\in\wlitp(A), J\models F(I,l)}w
& {} +
\mbox{$\sum$}_{(w:l)\in\wlitn(A), l\notin V, I\models l}w
\\
{} \leq {} &
\mbox{$\sum$}_{(w:l)\in\wlitp(A), I\models l}w
& {} +
\mbox{$\sum$}_{(w:l)\in\wlitn(A), l\notin V, I\models l}w
\\
{} \leq {} &
b
\text{,}
\end{array}
\end{equation*}
so that $J\not\models F(I,\gtA)$.
In turn,
$J\models F(I,\gtA)$ implies
$\ext(I,I)\models \gtA'$, and
$\ext(I,I)\models \ltA'$ follows from $J\models F(I,\ltA)$.
In view of
$J\models F(I,A)$ if and only if $J\models F(I,\gtA)$ or $J\models F(I,\ltA)$,
% we further conclude that
% We have thus shown that
% $I\models A$ and $J\models F(I,A)$ imply that
$J\models F(I,A)$ yields that
$\ext(I,I)\models \gtA'$ and (\ref{eq:p5}) or
$\ext(I,I)\models \ltA'$ and (\ref{eq:p6}) hold,
so that
$\ext(J,I)=J\cup\{\aux\}\cup\AtF(A,V)\subseteq J'$,
% so that
% $\{\aux\}\cup\AtF(A,V)\subseteq J'$ and
% $\ext(J,I)\subseteq\nolinebreak J'$.
while neither~(\ref{eq:p5}) nor~(\ref{eq:p6}) holds for
$J'=J\cup\{p^F\in\AtF(A,V) \mid p\notin J\}=\ext(J,I)$
when $J\not\models F(I,A)$.
In either case,
we have that
$\ext(J,I)\subseteq J'$ and
$\ext(J,I)\models F(\ext(I,I),\mon(A,V))$.
% In addition,
% $\ext(J,I)\models F(\ext(I,I),\mon(A,V))$
% follows from
% $\{p^F\in\nolinebreak\AtF(A,V) \mid p\notin J\}\subseteq \ext(J,I)$ along with
% $\ext(J,I)= J\cup\{\aux\}\cup \AtF(A,V)$
% when (\ref{eq:p5}) or (\ref{eq:p6}) holds.
\end{enumerate}
\end{proof}

\ThmCorrectness*
\begin{proof}
($\Rightarrow$) \
Let $I\in\SM(\Pi)$.
Then, by Lemma~\ref{LemReduct},
$\ext(I,I)\models F(\ext(I,I),\mon(A,V))$ and
$\ext(I,I)\subseteq I'$ as well as $\ext(I,I)\models F(I',\mon(A,V))$
for any model $I'$ of $\mon(A,V)$
such that $I'\setminus(\{\aux\} \cup \AtF(A,V)) = I$.
Since $I\models \Pi$ and $\aux\in\ext(I,I)$ if and only if $I\models A$,
this yields $\ext(I,I)\models\rew(\Pi,A,V)$ as well as $\ext(I,I)\models F(I',\rew(\Pi,A,V))$
for any model $I'$ of $\rew(\Pi,A,V)$
such that $I'\setminus(\{\aux\} \cup \AtF(A,V)) = I$,
so that $I'\in\SM(\rew(\Pi,A,V))$ implies $I'=\ext(I,I)$.

Let $J'\subset\ext(I,I)$ such that
$\ext(I,I)\setminus J'\subseteq C'$ for some component~$C'$
of $\rew(\Pi,A,V)$, and assume that
$J'\models (F(I,\Pi)\cap F(\ext(I,I),\rew(\Pi,A,V)))\cup F(\ext(I,I),\mon(A,V))$.
For $J:=J'\setminus(\{\aux\} \cup \AtF(A,V))$,
note that any path from $\alpha\in I\setminus J$ to $\beta\in I\setminus J$
in $\G_{\rew(\Pi,A,V)}$ that does not include $\aux$
is a path in $\G_\Pi$ as well,
while it maps to a path in $\G_\Pi$ (that includes~$A$) otherwise.
Hence, $I\setminus J\subseteq C$ for 
some component~$C$ of~$\Pi$,
so that $J'\models F(\ext(I,I),\mon(A,V))$
yields $J\subset I$ by Lemma~\ref{LemReduct}. % , where $J:=J'\setminus(\{\aux\} \cup \AtF(A,V))$.
Since $I\in\SM(\Pi)$, we have that $J\not\models F(I,\Pi)$,
while $J'\models F(I,\Pi)\cap F(\ext(I,I),\rew(\Pi,A,V))$
implies $J\models F(I,\Pi)\cap F(\ext(I,I),\rew(\Pi,A,V))$
because $\At(F(I,\Pi))\cap(\{\aux\} \cup \AtF(A,V))=\emptyset$. % no atom from $\{\aux\} \cup \AtF(A,V)$ occurs in $F(I,\Pi)$.
That is, $J\not\models F(I,\Pi)\setminus F(\ext(I,I),\rew(\Pi,A,V))$,
so that $I\models B(r)$, $J\models B(F(I,r))$, and $H(r)\cap J=\emptyset$ for some
rule $r\in \Pi\setminus \rew(\Pi,A,V)$.
For such a rule~$r$,
we have that $A\in B(r)$,
and $I\models r$ yields $H(r)\cap (I\setminus J)\neq\emptyset$.
Hence, by Lemma~\ref{LemReduct}, % we conclude that
$\ext(J,I)\subseteq J'$, where
$A\in B(r)$ together with
$I\models B(r)$ and
$J\models B(F(I,r))$ imply $\aux\in\ext(J,I)$.
This means that
$J'\models (B(F(I,r))\setminus \{A\})\cup\{\aux\}$,
while
$H(r)\cap J'=H(r)\cap J=\emptyset$, so that
$F(\ext(I,I),r')\in F(\ext(I,I),\rew(\Pi,A,V))$ and
$J'\not\models F(\ext(I,I),r')$
for the rule $r'\in\rew(\Pi,A,V)$ that replaces~$A$ in~$r$ by $\aux$.
We thus conclude that $J'\not\models F(\ext(I,I),\rew(\Pi,A,V))$ and
% $\ext(I,I)\in\SM(\rew(\Pi,A,V))$, and
$\{I'\in\nolinebreak\SM(\rew(\Pi,A,V)) \mid I'\setminus(\{\aux\} \cup \AtF(A,V)) = I\}
=\{\ext(I,I)\}$.

% \noindent
($\Leftarrow$) \
Let $I'\in\SM(\rew(\Pi,A,V))$ and
$I:=I'\setminus(\{\aux\} \cup \AtF(A,V))$.
Then, by Lemma~\ref{LemReduct}, we have that
$\ext(I,I)\subseteq I'$ and $\ext(I,I)\models F(I',\mon(A,V))$,
which yields 
$\ext(I,I)\models F(I',\rew(\Pi,A,V))$ and $I'=\ext(I,I)$.
Moreover,
$I\models \Pi$ holds because
$\At(\Pi)\cap(\{\aux\} \cup \AtF(A,V))=\emptyset$ % no atom from $\{\aux\} \cup \AtF(A,V)$ occurs in $\Pi$
and $\aux\in\ext(I,I)$ if and only if $I\models A$.

Let $J\subset I$ for some component~$C$ of~$\Pi$,
and assume that
$J\models F(I,\Pi)\cap F(\ext(I,I),\linebreak[1]\rew(\Pi,A,V))$.
For $J':=\ext(I,I)\setminus(I\setminus J)$,
since
$\ext(I,I)\in\SM(\rew(\Pi,A,V))$ and
$J'\subset \ext(I,I)$,
we have that
$J'\not\models F(\ext(I,I),\linebreak[1]\rew(\Pi,A,V))$,
while $J\models F(I,\Pi)\cap F(\ext(I,I),\linebreak[1]\rew(\Pi,A,V))$
implies $J'\models F(I,\Pi)\cap F(\ext(I,I),\linebreak[1]\rew(\Pi,A,V))$
because $\At(F(I,\Pi))\cap(\{\aux\} \cup \AtF(A,V))=\emptyset$
and $J'\setminus(\{\aux\} \cup \AtF(A,V)) = J$.
Moreover,
$\emptyset \subset H(r)\cap (\{\aux\} \cup \AtF(A,V)) \subseteq J'$
holds for any $r\in F(\ext(I,I),\mon(A,V))$,
so that
$J'\models F(\ext(I,I),\linebreak[1]\mon(A,V))$.
That is,
$J'\not\models F(\ext(I,I),\rew(\Pi,A,V))\setminus (F(I,\Pi)\cup\linebreak[1] F(\ext(I,I),\linebreak[1]\mon(A,V)))$,
so that
$\ext(I,I)\models B(r')$, $J'\models B(F(\ext(I,I),r'))$,
and $H(r')\cap J'=\emptyset$ for some
rule $r'\in \rew(\Pi,A,V)\setminus(\Pi\cup\mon(A,V))$.
For such a rule~$r'$,
we have that $\aux\in B(r')$,
and $\ext(I,I)\models r'$ yields $H(r')\cap (I\setminus J)\neq\emptyset$.
Thus,
$H(r)\cap (I\setminus J)\neq\emptyset$ and $A\in B(r)$ 
for the rule $r\in\Pi$ such that $r'$ replaces~$A$ in~$r$ by $\aux$.
Hence, by Lemma~\ref{LemReduct}, % we conclude that
$\ext(J,I)\models F(\ext(I,I),\linebreak[1]\mon(A,V))$,
but
$\ext(J,I)\not\models F(\ext(I,I),\linebreak[1]\rew(\Pi,A,V))$
in view of $\ext(J,I)\subset\ext(I,I)$.
Since
$\ext(J,I)\models F(I,\Pi)\cap F(\ext(I,I),\linebreak[1]\rew(\Pi,A,V))$
because $\At(F(I,\Pi))\cap(\{\aux\} \cup \AtF(A,V))=\emptyset$
and $\ext(J,I)\setminus(\{\aux\} \cup \AtF(A,V)) = J$,
this means that $\aux\in\ext(J,I)$,
$I\models A$, $J\models F(I,A)$,
$\ext(J,I)\not\models F(\ext(I,I),r')$, and
$J\not\models F(I,r)$ for rules $r'\in \rew(\Pi,A,V)\setminus(\Pi\cup\mon(A,V))$
and $r\in\Pi$ as above.
We thus conclude that $J\not\models F(I,\Pi)$ and
$I\in\SM(\Pi)$.
\end{proof}

\ThmModular*
\begin{proof}
The claim follows immediately by observing that
$\rew(\Pi,A,\V)\cup\Pi'\subseteq\rew(\Pi \cup \Pi',A,\V)$
and 
$\rew(\Pi \cup \Pi',A,\V)\setminus\rew(\Pi,A,\V)\subseteq\Pi'$.
\end{proof}

%%% Local Variables: 
%%% mode: latex
%%% TeX-master: "main"
%%% End: 

\fi

\end{document}